\documentclass[11pt]{article}

\usepackage[left=1in,right=1in,top=1in,bottom=1in]{geometry}
\usepackage{setspace} % [onehalfspacing]
\usepackage{mathpazo}
\usepackage[pagebackref=true]{hyperref}
\usepackage{url}

\usepackage{booktabs}       % professional-quality tables
\usepackage{amsfonts}       % blackboard math symbols
\usepackage[nice]{nicefrac} % compact symbols for 1/2, etc.
\usepackage{microtype}      % microtypography

\usepackage{textcomp}
\usepackage{xcolor}
\usepackage{amsmath,amsthm,amstext,amsfonts,amssymb,mathrsfs}
\usepackage{graphicx,caption,subcaption,color,algorithm,algorithmic} 

\usepackage{wrapfig,dsfont,booktabs}

\usepackage[capitalize]{cleveref}
\usepackage{marvosym}
\usepackage{changepage}
\usepackage{parskip}

\newtheorem{lemma}{Lemma}

\newtheorem{assumption}{Assumption}

\newtheorem{proposition}{Proposition}
\newtheorem{theorem}{Theorem}
\newtheorem{definition}{Definition}
\newtheorem{corollary}{Corollary}

\newcommand{\dd}{\mathrm{d}\,}
\newcommand{\E}{\mathbb{E}}

\newcommand{\R}{\mathbb{R}}
\newcommand{\cS}{\mathcal{S}} 
\newcommand{\cA}{\mathcal{A}}
\newcommand{\ScA}{\cS\times\cA}
\newcommand{\ScS}{\cS\times\cS}
\newcommand{\cP}{\mathcal{P}}
\newcommand{\cPhat}{\widehat{\mathcal{P}}}

\newcommand{\cV}{\mathcal{V}}

\newcommand{\cD}{\mathcal{D}}
\newcommand{\cO}{\mathcal{O}}
\newcommand{\cX}{\mathcal{X}}
\newcommand{\cM}{\mathcal{M}}

\newcommand{\cE}{\mathcal{E}}

\newcommand{\expect}[1]{\mathbb{E}\left[ #1 \right]}
\newcommand{\expnew}[1]{\mathrm{exp}\left( #1 \right)}

\newcommand{\cOtilde}{\widetilde{\cO}}
\newcommand{\cMhat}{\widehat{\cM}}
\newcommand{\cardS}{\ensuremath{|\cS|}}
\newcommand{\cardA}{\ensuremath{|\cA|}}
\newcommand{\prob}[1]{\mathbb{P}\left( #1 \right)}
\newcommand{\mean}[1]{\overline{#1}}
\newcommand{\abs}[1]{\left\lvert #1 \right\rvert}

\newcommand{\nn}{\nonumber}

\newcommand{\hatpi}{\widehat{\pi}}
\newcommand{\pihat}{\widehat{\pi}}

\newcommand{\nuhat}{\widehat{\nu}}
\newcommand{\Phat}{\widehat{P}}
\newcommand{\Ptilde}{\widetilde{P}}

\newcommand{\Vhat}{\widehat{V}}

\newcommand{\tv}{\mathrm{tv}}

\newcommand{\kl}{\mathrm{kl}}
\newcommand{\chisq}{\mathrm{c}}
\newcommand{\wass}{\mathrm{w}}

\newcommand{\dtv}{D_{\mathrm{TV}}}

\newcommand{\dkl}{D_{\mathrm{KL}}}
\newcommand{\dwass}{D_{\mathrm{w}}}

\newcommand{\dchi}{D_{\mathrm{c}}}
\newcommand{\ipmV}{\mathrm{d}_{\cV}}

\DeclareMathOperator*{\argmax}{arg\,max}

\newcommand{\norm}[1]{\left\| #1 \right\|}% 
\newcommand{\trace}{\mathrm{Tr}}
\newcommand{\rank}{\mathrm{rank}}
\newcommand{\drectsuffcov}{C^\dagger_{\mathrm{sc}}}
\newcommand{\jinsuffcov}{C_{\mathrm{sc}}}
\newcommand{\relativecondition}{C_{\pi^*,\phi}}
\newcommand{\drectrelativecondition}{C^{\dagger}_{\pi^*,\phi}}

\allowdisplaybreaks
\usepackage[round]{natbib}
\renewcommand{\bibname}{References}
\renewcommand{\bibsection}{\subsubsection*{\bibname}}

\newcommand{\indic}{\mathds{1}}
\DeclareMathOperator*{\Binomial}{Binomial}

\title{\singlespace{ {\Large Bridging Distributionally Robust Learning and Offline RL: \\An Approach to Mitigate Distribution Shift and Partial Data Coverage}}}

\author{Kishan Panaganti$^{1}$\thanks{ This work was done when the corresponding author, Kishan, was a PhD candidate at Texas A\&M University. } ,\, Zaiyan Xu$^{2}$,\, Dileep Kalathil$^{2}$,\, Mohammad Ghavamzadeh$^{3}$ \\
  $^{1}$\,California Institute of Technology, $^{2}$\,Texas A\&M University,  $^{3}$\,Amazon \\Emails:  \texttt{kpb@caltech.edu,\{zxu43,\,dileep.kalathil\}@tamu.edu}, \texttt{ghavamza@amazon.com}
}

\begin{document}

\maketitle

\begin{abstract}
The goal of an offline reinforcement learning (RL)  algorithm is to learn optimal  polices   using  historical (offline) data, without   access to the environment for online exploration. One of the main challenges in offline RL is the distribution shift which refers to  the difference between the state-action visitation distribution of the data generating policy and the learning policy.  Many recent works have used the idea of pessimism for developing offline  RL algorithms and characterizing  their sample complexity  under a relatively weak assumption of single policy concentrability. Different from the offline RL literature, the area of distributionally robust learning (DRL) offers a principled framework that uses a minimax formulation to tackle model mismatch between training and testing environments. In this work, we aim to bridge these two areas by showing that the DRL approach can be used to tackle the distributional shift problem in offline RL. In particular, we propose two offline RL algorithms using the DRL framework,  for the tabular and linear function approximation settings, and characterize their sample complexity under the single policy concentrability assumption. We also demonstrate the superior performance our proposed algorithm through simulation experiments. 
\end{abstract}

\section{Introduction}
\label{sec:introduction-mpqi}

The  goal of an offline RL algorithm is to learn an approximately optimal  policy using minimal amount of  offline data collected according to a behavior policy \citep{lange2012batch, levine2020offline}.  The lack of online  exploration makes the offline RL problem particularly challenging due to \emph{distribution shift} and \emph{partial data coverage}.  Distribution shift refers to the difference between the state-action visitation distribution of the behavior policy and that of the learned policy.  Partial data coverage refers to the fact that the data generated according to the behavior policy may only contain samples from parts of the state-action spaces. While these two issues are not the same, in effect, they both cause the problem of out-of-distribution (OOD) data \citep{yang2021generalized, robey2020model}, i.e., distributions of training and testing data being different.

In the past few years, many works have developed deep offline RL algorithms mitigating distribution shift and partial data coverage, but have been mainly focused on the algorithmic and empirical aspects \citep{fujimoto2019off, kumar2019stabilizing, kumar2020conservative,  fujimoto2021minimalist, kostrikov2021offline}. Most of the early theoretical works on offline RL however analyzed the performance of their algorithms by making the strong assumption of \emph{uniformly bounded concentrability} which requires that the ratio of the state-action occupancy distribution induced by \emph{any} policy and the data generating distribution being bounded uniformly over all states and actions \citep{munos2007performance, antos2008learning, munos08a, farahmand2010error, chen2019information, liao2022batch}. The more recent theoretical results have used the principle of pessimism or conservatism \citep{yu2020mopo, buckman2021the, jin2021pessimism} and addressed some of the issues in offline RL, including replacing uniform concentrability with the more relaxed \emph{single policy concentrability} assumption \citep{uehara2021pessimistic, rashidinejad2022bridging, li2022settling}.  

\subsection{Motivation: Why Distributionally Robust Learning for Offline RL?}

Classical supervised learning is based on empirical risk minimization (ERM), which assumes that the train and test data are drawn from the same distribution \citep{shalev2014understanding}. However, this assumption is hardly satisfied in many real-world applications \citep{quinonero2022dataset}, and the performance of supervised learning algorithms degrade significantly in the out-of-distribution setting \citep{taori2020measuring, koh2021wilds}. A large body of work has been recently developed that uses the distributionally robust learning (DRL) framework to address the issue of distribution shift in various settings \citep{duchi2018learning, kuhn2019wasserstein, chen2020distributionally}. The DRL framework considers an uncertainty  set of data distributions around a nominal distribution (typically the training data distribution), and solves a minimax optimization problem to find a function that minimizes the expected loss, where the expectation is taken w.r.t.~the distribution in the uncertainty set that maximizes the loss. 
DRL is a principled framework that provides generalization guarantees, accommodates ways of constructing domain specific uncertainty sets (e.g.,~using $f$-divergence and Wasserstein distance), and offers practical and scalable algorithms \citep{chen2020distributionally, levy2020large, esfahani2015data}.

The issue of out-of-distribution data arises in real-world RL applications because of the mismatch between the train and test environments (MDP models). This issue is also known as simulation-to-reality (sim-to-real) gap \citep{tobin2017domain}. RL algorithms are typically trained using a simulator (online RL) or a pre-collected offline dataset (offline RL). However,  modeling errors and changes in the real-world system parameters are inevitable in RL applications, and standard  RL policies can fail dramatically even when they face a mild mismatch between the train and test environments \citep{tobin2017domain,peng2018sim}. Many works have used the heuristic  of domain randomization \citep{weng2019DR} to make the learned RL policy robust against sim-to-real gap. More recently, a number of works have proposed to use the DRL framework in RL, building on the formalism of robust Markov decision processes (RMDPs) \citep{iyengar2005robust,nilim2005robust} and adapting ideas from the supervised learning counterpart. Here are instances of value-based \citep{tamar2014scaling,roy2017reinforcement,panaganti2020robust,panaganti-rfqi,panaganti22a,xu-panaganti-2023samplecomplexity, wang2021online,ma2022distributionally} and policy-based 
\citep{wang2022policy,kumar2023policy,li2022first,wang2022convergence,grand2021scalable} distributionally robust RL (DRRL) algorithms with provable performance guarantees. However, these  works do not consider the offline RL setting in which the out-of-distribution issues are due to the distribution shift and partial data coverage.

Offline RL closely resembles supervised learning because its goal is to learn a policy from an offline dataset, as opposed to the conventional RL goal of learning through online exploration. As a result, it faces similar out-of-distribution issues as in supervised learning. As mentioned above, DRL has shown to be an attractive framework to address the out-of-distribution issues arising in supervised learning problems, offering practical algorithms with provable performance guarantees. These observations motivate us to ask the following questions:
\begin{quote} 
    \textit{Can we address the distributional shift issues in offline RL using distributionally robust learning as a principled approach? What kind of theoretical performance guarantees can we provide and under what kind of assumptions?} 
\end{quote}
 In this work, we answer these questions affirmatively. In particular, we propose  offline RL algorithms using the framework of DRL for the tabular and linear MDP settings, and characterize their sample complexity. Moreover, we show that our  approach enables the relaxation of the strong assumption of  uniform concentrability to  single policy concentrability.

\begin{table*}[t]
\small \begin{adjustwidth}{-1em}{}
\begin{center}
    \begin{tabular}{cccc}
\hline
Algorithm                 & Algorithm-type          & Data coverage assumption & Suboptimality                                           \\ \hline
Lower bound &&& \\ \citep[Th.7]{rashidinejad2022bridging}                & -                       & single-policy    & $\cOtilde \left(\sqrt{\frac{\cardS (C_{\pi^*}-1)}{(1-\gamma)^3 N}}\right)$    \\ \hline
\citep[Th.6]{rashidinejad2022bridging}           & reward pessimism                 & single-policy & $\cOtilde \left(\sqrt{\frac{\cardS C_{\pi^*}}{(1-\gamma)^5 N}}\right)$   
\\ \hline
\citep[Th.1]{li2022settling}        & reward   pessimism                     & single-policy,  clipped    & $\cOtilde \left(\sqrt{\frac{\cardS C_{\pi^*},\text{clip}}{(1-\gamma)^3 N}}\right)$  
\\ \hline 
\citep[Cor.1]{uehara2021pessimistic}           & oracle model   pessimism         & single-policy  & $\cOtilde \left(\sqrt{\frac{\cardS^2 \cardA C_{\pi^*}}{(1-\gamma)^4 N}}\right)$
\\ \hline
DRQI   (this work, Th.1)            &  distributionally robust                      & single-policy   & $\cOtilde \left(\sqrt{\frac{\cardS^2 C_{\pi^*}}{(1-\gamma)^4 N}}\right)$
\\ \hline
\end{tabular}
\end{center}
\caption{Comparison of the offline RL algorithms in the tabular setting. The data coverage assumption is based on the single-policy concentrability $C_{\pi^*} = \max_{s,a} ({ d^{\pi^*}(s,a)}/{\mu(s,a)})$ and its clipped version $C_{\pi^*,\text{clip}} = \max_{s,a}( {\min\{ d^{\pi^*}(s,a),1/\cardS \}}/{\mu(s,a)})$, where $d^{\pi^*}$ is the discounted occupancy measure of the optimal policy $\pi^{*}$ and $\mu$ is the state-action visitation distribution of the  data generating policy. The suboptimality column is the statistical bounds for the offline RL objective (\cref{eq:offline-rl-objective}), where $\cardS$ and $\cardA$ are the number of states and actions, $\gamma$ is the discount factor, and $N$ is the size of the offline data.} 
\label{tbl:tabular-offline-rl-results-compare}
\end{adjustwidth}
\end{table*}

\subsection{Comparisons and Contributions}

We outline our contributions and compare our theoretical results with several recent works that, similar to us, only use the single concentrability assumption. 

\citet{uehara2021pessimistic} propose a pessimistic model-based offline RL algorithm, which we refer to as \textit{oracle model pessimism} in \cref{tbl:tabular-offline-rl-results-compare} and \cref{tbl:linear-mdp-offline-rl-results-compare}. While their proposed algorithm is similar to the max-min formulation of DRL, they do not offer a computationally tractable implementation for it. It is known in the RMDP literature \citep{iyengar2005robust,nilim2005robust,wiesemann2013robust} that solving the max-min objective (\cref{eq:robust-value-function}) can be NP-hard without additional structural assumptions, such as \textit{rectangularity}. \citet{rashidinejad2022bridging} propose a lower confidence bound  algorithm based on the idea of pessimism in the face of uncertainty. The algorithm subtracts a pessimistic term from the reward estimate, and hence we call it \textit{reward pessimism} in \cref{tbl:tabular-offline-rl-results-compare}. They also provide a lower-bound on the sample complexity of offline RL algorithms. \citet{li2022settling} also propose a reward pessimism-based offline RL algorithm. They use a more sophisticated analysis and obtain a sample complexity guarantee that matches the lower-bound. They are also able to use an improved clipped concentrability coefficient which is less than the single policy concentrability used in other works. We note that  \cite{rashidinejad2022bridging} and \cite{li2022settling} only study the tabular setting. In the linear function approximation setting, the state-of-the-art algorithms are based on \textit{reward pessimism} and their sample complexity guarantees depend on the linear feature dimension, as opposed to state and action space dimensions in the tabular setting \citep{jin2021pessimism,yin2022near,xiong2022nearly}.

\begin{table*}[t]
\small \begin{adjustwidth}{-2em}{}
\begin{center}
\begin{tabular}{cccc}
\hline
Algorithm                 & Algorithm-type          & Data coverage assumption & Suboptimality                                           \\ \hline
\cite[Cor.4.5]{jin2021pessimism}        & reward pessimism                        & w.h.p $\Lambda_N \geq I/N + \jinsuffcov \cdot \Sigma_{d^{\pi^*}}$     & $\frac{d\sqrt{\rank(\Sigma_{ d^{\pi^*}})} }{\sqrt{\jinsuffcov(1-\gamma)^4 N}}$  
\\ \hline
\cite[Th.6]{uehara2021pessimistic}           & oracle model   pessimism            &  $\relativecondition<\infty$    & $\sqrt{\frac{ \rank(\Lambda)^2 d \relativecondition}{(1-\gamma)^4 N}}$
\\ \hline
  LM-DRQI  (this work, Th.2)      & distributionally robust                          & $\forall i\in[d]$ w.h.p $\Lambda_N \geq I/N + \drectsuffcov d \cdot \Sigma^{i}_{d^{\pi^*}}$   & $\frac{\sqrt{\rank(\Sigma_{ d^{\pi^*}})d} }{\sqrt{\drectsuffcov(1-\gamma)^4 N}}$ 
  % \\\hline
  % LM-DRQI  (this work, Cor.1)      & distributionally robust                          & $\drectrelativecondition<\infty$   & $\frac{\sqrt{d \drectrelativecondition \rank(\Lambda)^2}}{\sqrt{(1-\gamma)^4 N}}$
\\ \hline
\end{tabular}
\end{center}
\caption{ Comparison of the offline RL algorithms in the linear MDP setting. Here, $\Sigma_{ d^{\pi^*}}=\E_{s,a\sim d^{\pi^*}}[\phi(s,a)\phi(s,a)^\top]$,  $\Lambda =  \E_{s,a\sim\mu}[\phi(s,a)\phi(s,a)^\top]$, $\Lambda_N$ is an estimate of $\Lambda$ (see \cref{eq:lambda-N}),  $\relativecondition = \max_{x\in\R^d} { (x^\top \Sigma_{ d^{\pi^*}} x)}/{(x^\top \Lambda x)}$, 
%$ \drectrelativecondition  =  \max_{x\in\R^d} \sum_{i\in[d]} { d (x^\top \Sigma^{i}_{ d^{\pi^*}} x)}/{(x^\top \Lambda x)}$, 
$\Sigma^{i}_{ d^{\pi^*}}=\E_{s,a\sim d^{\pi^*}}[(\phi_i(s,a)\indic_i)(\phi_i(s,a)\indic_i)^\top]$, $\indic_i$ is the unit vector in $i$th dimension,  $\phi(s,a)\in\R^d$ is $d$-dimensional feature vector, and   $\jinsuffcov$ and $\drectsuffcov$ are the sufficient coverage constants satisfying corresponding random but high probability events. 
}
\label{tbl:linear-mdp-offline-rl-results-compare}
\end{adjustwidth}
\end{table*}

\textbf{Our Contributions:} $(i)$ We propose a novel offline RL algorithm using the DRL framework, called Distributionally Robust Q-Iteration (DRQI), for the tabular setting. We show that our  approach is able to relax the strong assumption of uniform concentrability to a weaker single policy concentrability assumption. We also provide detailed analysis and sample complexity results for DRQI with four commonly used uncertainty sets in DRL: total variation, Wasserstein, Kullback-Leibler, and chi-square uncertainty sets. The comparison with the relevant works is given in  \cref{tbl:tabular-offline-rl-results-compare}.  \\ 
$(ii)$ We extend our distributionally robust approach to offline RL to the linear MDP setting, propose the Linear MDP DRQI (LM-DRQI) algorithm. We   characterize its sample complexity using only the \emph{sufficient coverage} assumption \citep{jin2021pessimism} which only requires that the  trajectory induced by the optimal policy $\pi^{*}$ is  covered by the offline data  sufficiently well. In particular, we do not require the uniform concentrability assumption. The comparison with the  relevant works is given in \cref{tbl:linear-mdp-offline-rl-results-compare}. \\
$(iii)$ We demonstrate the performance of DRQI algorithm through simulation experiments. In the partial data coverage setting, DRQI algorithm performs better than the standard dynamic programming approach, and performs  at par with the state-of-the-art reward pessimism based offline RL algorithms. In the full coverage setting,  DRQI algorithm outperforms the reward pessimism based offline RL algorithms. \\
$(iv)$ We believe that establishing a connection between the DRL and offline RL literature is also a contribution of this work. It provides the opportunity for bringing the machinery from DRL to solve the offline RL problem. In particular, we expect that the offline RL problems with large state and action spaces could greatly benefit from this.

We note that our sample complexity result  is $\cO(\sqrt{\cardS/(1-\gamma)})$ away from the state-of-the-art lower-bound (and the matching upper-bound) in the tabular setting (c.f. \cref{tbl:tabular-offline-rl-results-compare}). 
We, however, believe that our result can be improved using the more sophisticated  variance-based  concentration arguments as used in \cite{li2022settling}. This analysis is more challenging for the distributional robust setting and we defer that to future work.
In the linear MDP setting, our result is comparable to \citet{jin2021pessimism} as long as  $\jinsuffcov \leq d \drectsuffcov$. 
However, for a certain class of linear MDPs \citet{jin2021pessimism}'s data coverage assumption implies ours (c.f.\cref{lem:suff-cov-equivalence-class}) and hence $\jinsuffcov=\drectsuffcov$, our result improves over \citet{jin2021pessimism} by $\sqrt{d}$.
Our result is not directly comparable with that of \citet{uehara2021pessimistic}. We also want to emphasize \citet{uehara2021pessimistic} does not provide a tractable implementation. 
However, from the linear MDP problem setup, our LM-DRQI algorithm can use the least squares regression prescription from \citet{ma2022distributionally} for implementation.

\textbf{Comparison with \citet{wang2023distributionally}:} In the final stages of working on this manuscript we came across the work by \citet{wang2023distributionally}, who propose a similar offline RL algorithm as ours (\cref{alg:MPQI-Algorithm}). \cite{wang2023distributionally} only consider  the tabular setting, whereas we provide offline RL algorithms for both the tabular and linear MDP settings.  \cite{wang2023distributionally} consider a total variation uncertainty set whereas we consider four  commonly used uncertainty sets in DRL. In terms of the sample complexity guarantees, they provide a $\cOtilde (\sqrt{({\cardS C_{\pi^*}^- })/({(1-\gamma)^4 N})})$ bound. However, we want to point out that there is a technical error in their application of Hoeffding's inequality to $L^1$-norm \citep[Eq.(10)]{wang2023distributionally}. To emphasize, Hoeffding's inequality (\cref{thm:hoeffding}) gives a concentration result for \textit{single-valued random variables}, hence we incur  an additional $\cardS$ factor in the concentration of total variation distance (equivalently for $L^1$-norm) between two \textit{random vectors}. This observation matches the tightness of concentration of empirical distributions under total variation distance \citep[Theorem 1]{canonne2020short}. This technical error makes their bound appear $\sqrt{\cardS}$ better than it should be. If this error is fixed, then their sample complexity results will match ours. \citet{wang2023distributionally} also derive an improved bound using the Bernstein-based analysis techniques \citep{li2022settling}. Although this bound is optimal, it is only when the sample size $N$ exceed $\cOtilde( 1/((1-\gamma) \mu_{\min}^2))$, where $\mu$ is the data generating distribution and $\mu_{\min}$ is its minimal positive value. Hence they  get quadratic dependence on  $\cardS$ and $\cardA$ for sample complexity, but also note $C_{\pi^*,\text{clip}} \leq \cardA$, when $\mu$  is a  uniform  distribution.
Nonetheless, we want to emphasize that the analysis in \citet{wang2023distributionally} are {sophisticated and insightful}. We believe both works  make interesting contributions to offline RL literature.

\section{Preliminaries}
\label{sec:formulation-mpqi}

\textbf{Notations:} For a set $\cX$, we denote its cardinality as $|\cX|$. The set of probability distributions over $\cX$  is denoted  as $\Delta(\cX)$.
For any vector $x$ and positive semidefinite matrix $A$, $\|x\|_{A} = \sqrt{x^{\top} A x}$. Let $\trace(\cdot)$ denote the trace operator. Denote $\indic_i\in\{0,1\}^{d\times 1}$ as a zero-vector with value $1$ at index $i$.
We use $f\leq\cO(g)$ to denote $f\leq c\cdot g$ for some universal constants less than $100$, and likewise use $f\leq\cOtilde(g)$ to absorb all the universal constants less than $100$ and the polylog terms depending on $d,N$ and $1/(1-\gamma)$.

\textbf{Markov Decision Process (MDP):} An MDP is a tuple $(\cS,\cA,r,P^o,\gamma,d_0)$, where $\cS$ is the  state space, $\cA$ the action space, $r:\ScA\to[0,1]$ is the  reward function, $P^o: \ScA \to \Delta(\cS)$ is the probability transition function (model), $\gamma$ is the discount factor, and $d_0$ is the initial state distribution. A stationary (stochastic) policy $\pi: \cS \to \Delta(\cA)$ specifies a distribution over actions for each state. Each policy $\pi\in\Pi$ induces a discounted occupancy distribution over state-action pairs, denoted as $d^\pi: \cS \times \cA \to [0,1]$, where $d^\pi(s,a) = (1-\gamma)\sum_{t=0}^\infty \gamma^t P_t(s_t = s, a_t = a; \pi)$, and $P_t(s_t = s, a_t = a; \pi)$ denotes the visitation probability of state-action pair $(s,a)$ at time step $t$, starting at $s_0 \sim d_0(\cdot)$ and following $\pi$ on the model $P^o$. For simplicity, we denote $P_t(s_t = s, a_t = a; \pi)$ by $d^\pi_t(s,a)$. The value of a policy $\pi$ at state $s \in \cS$ is $V^\pi_{P^o}(s) = \E_{\pi,P^o} [\sum_{t=0}^\infty \gamma^t r(s_t,a_t) \; | \; s_0 = s]$, where $a_t \sim \pi(\cdot |s_t)$ and $s_{t+1} \sim P^o_{ s_t,a_t}$. Similarly, we define the  $Q$-value of a policy as  $Q^\pi_{P^o}(s,a) = \E_{\pi,P^o} \left[\sum_{t=0}^\infty \gamma^t r_t \; | \; s_0 = s , a_0 = a \right].$ We sometimes denote $d^\pi$ as $d^\pi_{P^o}$ making its dependence on the model $P^o$ clearer.

\paragraph{Offline RL:} In offline RL, we only have access to a pre-collected offline dataset consisting of $N$ samples: $\cD = \{(s_i,a_i,r_i,s_i')\}_{i=1}^N$, where $r_i = r(s_i,a_i)$ and $s_i' \sim P^o_{ s_i,a_i}$. We assume that $(s_i,a_i)$ pairs are generated i.i.d.~by following a data generating (behavior) distribution $\mu \in \Delta(\cS \times \cA)$.  The {goal of offline RL} is to learn a \textit{good} policy $\hatpi$ close to an optimal policy $\pi^*$ of MDP $M^o$ based on the offline data $\cD$. More formally, for a prescribed accuracy level $\epsilon$, we seek to find an $\epsilon$-optimal policy $\hatpi$ satisfying
\begin{align} \label{eq:offline-rl-objective}
    \E_{s_0\sim d_0} [V^{\pi^*}(s_0) - \E_\cD[{V}^{\hatpi}(s_0)]] \leq \epsilon,
\end{align}
with high probability using an offline dataset $\cD$  containing as few samples as possible.

Analysis of offline RL algorithms   crucially depends on the \emph{data coverage} assumption,  which is  quantified using the \emph{concentrability
coefficient}. For a given policy $\pi$, the concentrability coefficient $C_{\pi}$ is defined as 
   $ C_{\pi} = \max_{(s,a) \in \cS \times \cA} {d^{\pi}(s,a)}/{\,\mu(s,a)}. $
Most of the past theoretical works on offline RL  use the strong assumption of bounded \emph{uniform concentrability} \citep{munos08a}, defined as  $C_{u} = \sup_{\pi} C_{\pi}$.
\citet{munos08a} propose fitted Q-iteration algorithm and give offline RL guarantees under uniform concentrability.
Recently, some  works have proposed offline RL algorithms using the idea of pessimism and showed that the   {uniform concentrability}  can be relaxed to a \emph{single concentrability} assumption, i.e.,  $C_{\pi^{*}}$ is bounded \citep{uehara2021pessimistic, rashidinejad2022bridging, li2022settling}. We also make the same single concentrability assumption in this work.

\textbf{Robust Markov Decision Process (RMDP):}  The RMDP formulation considers a set of models called uncertainty set,  denoted as $\cP$.  We assume that  $\cP$  satisfies the standard \textit{$(s,a)$-rectangularity condition} \citep{iyengar2005robust}. An RMDP can be specified as  $(\cS, \cA, r, \cP, \gamma, d_0)$ in which
\begin{align}
    \label{eq:uncertainty-set-1}
    \mathcal{P} &= \otimes_{(s,a) \in \ScA }\, \mathcal{P}_{s,a},  \\
    \label{eq:uncertainty-set-2}
    \mathcal{P}_{s,a} &=  \{ P_{s,a} \in \Delta(\cS)~:~ D(P_{s,a}, P^o_{s,a}) \leq \rho_{s,a}  \},
\end{align}
where $D(\cdot, \cdot)$ is a distance metric between two probability distributions and $\rho_{s,a} > 0$ is the radius of the uncertainty set. In other words, $\mathcal{P}$ is the set of all models around $P^{o}$ within a particular distance.

The \textit{robust value function} $V^{\pi}_{\mathcal{P}}$ corresponding to a policy $\pi$ and the \textit{optimal robust value function} $V^{*}_{\mathcal{P}}$ are defined as \citep{iyengar2005robust,nilim2005robust}
\begin{align}
\label{eq:robust-value-function}
V^{\pi}_{\mathcal{P}} = \inf_{P \in \mathcal{P}} ~V^{\pi}_{P},\qquad V^{*}_{\mathcal{P}} = \sup_{\pi} \inf_{P \in \mathcal{P}} ~V^{\pi}_{P} . 
\end{align} 
An \textit{optimal robust policy} $\pi^{*}_{\mathcal{P}}$ is such that the robust value function corresponding to it matches the optimal robust value function, i.e., $V^{\pi^{*}_{\mathcal{P}}} = V^{*}_{\mathcal{P}} $. It is known that there exists a stationary and deterministic optimal policy \citep{iyengar2005robust} for the RMDP. The \textit{robust Bellman operator} is defined  as \citep{iyengar2005robust}
\begin{align}
\label{eq:robust-bellman-primal}
    (T Q)(s, a) = r(s, a) + \gamma \inf_{P_{s,a} \in \cP_{s,a}} \!\!\!\E_{s' \sim P_{s,a}} [ \max_{b} Q(s', b)]. 
\end{align}
It is known that $T$ is a contraction mapping in the infinity norm and hence it has a unique fixed point $Q^{*}_{\mathcal{P}}$ with $V^{*}_{\mathcal{P}}(s) = \max_{a} Q^{*}_{\mathcal{P}}(s,a)$ and $\pi^{*}_{\mathcal{P}}(s) = \argmax_{a} Q^{*}_{\mathcal{P}}(s,a)$ \citep{iyengar2005robust}. The robust Q-Iteration   can now be defined using the robust Bellman operator as  $Q_{k+1} = T Q_{k}$. Since $T$ is a contraction, it follows that $Q_{k} \rightarrow Q^{*}_{\mathcal{P}}$. So, robust Q-Iteration  can be used to compute (solving the planning problem) $Q^{*}_{\mathcal{P}}$ and $\pi^{*}_{\mathcal{P}}$ in the tabular setting with a known uncertainty set $\cP$.

\section{Distributionally Robust Q-Iteration (DRQI) Algorithm}
\label{sec:mpqi}

In this section, we  propose our DRQI algorithm to solve the offline RL problem in the tabular setting and provide its theoretical guarantees.

Let $N(s,a) = \sum^{N}_{i=1} \indic\{(s_{i}, a_{i}) = (s,a)\}$ and  $N(s,a,s') = \sum^{N}_{i=1} \indic\{(s_{i}, a_{i}, s'_{i}) = (s,a, s')\}$ . We  then construct an empirical estimate of $P^{o}$ as 
\begin{equation*}
    \Phat^{o}_{s,a}(s') = \frac{N(s,a,s')\indic\{N(s,a)\geq1\}}{N(s,a)}+ \frac{\indic\{N(s,a)=0\}}{|\cS|}.
\end{equation*}
We also consider the add-$L$  estimate \citep{bhattacharyya2021near,arora2023near} of $P^{o}$ given by
\begin{align*}
    \Ptilde^{o}_{s,a}(s') = \frac{N(s,a,s')+L}{N(s,a)+L\cardS},
\end{align*}
where the value of $L$ is defined later. Following the uncertainty set definition (c.f. \cref{eq:uncertainty-set-1}-\cref{eq:uncertainty-set-2}), we  construct the empirical uncertainty set $\cPhat$ around $\Phat^{o}$ or $\Ptilde^{o}$ as, $\cPhat = \bigotimes_{s,a} \cPhat_{s,a}$, where 
\begin{align}
\label{eq:uncertainty-set}
\cPhat_{s,a} = \{ P\in \Delta(\cS) : D(P,\Phat^o_{s,a} \text{ or } \Ptilde^o_{s,a}) \leq \rho_{s,a} \}.
\end{align}
Similarly (c.f. \cref{eq:robust-bellman-primal}), we can define the empirical robust Bellman operator  as 
\begin{align}
     (\widehat{T} Q)(s, a) = r(s, a) + \gamma \inf_{P_{s,a} \in \cPhat_{s,a}} \!\!\!\E_{s' \sim P_{s,a}} [ \max_{b} Q(s', b)]. \label{eq:robust-bellman-eq}
\end{align}
Note that for $\rho_{s,a} = 0$,  $\widehat{T}$ is the same as the standard (non-robust) empirical Bellman operator. Thus, the empirical Q-value iteration $Q_{k+1} = \widehat{T} Q_{k}$ will give an approximately optimal Q-value function under the standard generative model assumption where there are $N(s,a)=N$ next-state samples from each $(s,a)$ pairs \citep{haskell2016empirical, kalathil2021empirical}. However, since  the data is generated according to a behavior policy in the offline RL,, the generative model assumption is not valid here.  On the other hand, for a fixed $\rho_{s,a} > 0,$ the update $Q_{k+1} = \widehat{T} Q_{k}$ is exactly equal to empirical robust Q-iteration, and it will converge to an approximately optimal robust Q-function corresponding to the RMDP uncertainty set specified by the $\rho_{s,a}$ values \citep{panaganti22a,xu-panaganti-2023samplecomplexity, shi2022distributionally}

The key insight of our algorithm is to use the update $Q_{k+1} = \widehat{T} Q_{k}$ as a DRL style approximate Q-iteration. To see this, recall the standard  DRL problem \citep{duchi2018learning, chen2020distributionally}:
 $   \max_{\theta} \; \min_{q \in \mathcal{Q}} \; \mathbb{E}_{x \sim q}[f(x;\theta)],
$
where $f$ is a function to be maximized w.r.t.~a parameter $\theta$ and $\mathcal{Q}$ is an uncertainty set for the probability distribution. The nomenclature `distributionally robust' is due to the term $\min_{q \in \mathcal{Q}}$ in the objective. Now, in our case,  the minimization over the uncertainty set $\widehat{\cP}$ in the definition of $ \widehat{T}$, i.e.,~$\inf_{P_{s,a} \in \widehat{\cP}_{s,a}}$, also represents this distributionally robust objective.  Observing that the degree of the robustness  depends on the radius of the uncertainty set $\rho_{s,a}$, we propose to control this robustness by choosing an appropriate value for $\rho_{s,a}$ depending on the offline data $\mathcal{D}$. In particular, we will choose $\rho_{s,a} = \min \left(c_{1}, c_{2}/\sqrt{N(s,a)} \right)$, where $c_{1}$ and $c_{2}$ are problem-dependent constants to be specified later.

Throughout our analysis, we assume that the reward function is known to the algorithm, in order to focus on the key DRL idea due to the term ~$\inf_{P_{s,a} \in \widehat{\cP}_{s,a}}$. This relaxation is made without loss of generality since we can model similar uncertainty sets $\cP$ or $\cPhat$ for the reward distributions \citep{si2020distributionally,zhou2021finite}.

\begin{algorithm}[t]
	\caption{Distributionally Robust Q-Iteration (DRQI) Algorithm}	
	\label{alg:MPQI-Algorithm}
	\begin{algorithmic}[1]
		\STATE \textbf{Input:} Offline data $\cD =(s_i,a_i,r_{i}, s'_{i})_{i=1}^N$, Confidence level $\delta\in(0,1)$ 
		\STATE \textbf{Initialize:} $Q_{0}\equiv 0$
		\FOR {$k=0,\cdots,K-1$ } 
        \STATE  Compute  $Q_{k+1} = \widehat{T} Q_{k}$ from \cref{eq:robust-bellman-eq}
		\ENDFOR
		
		\STATE \textbf{Output:} $\pi_{K}  = \argmax_a Q_{K}(s,a)$
	\end{algorithmic}
\end{algorithm}

In this work, we consider four uncertainty sets corresponding to four different distance metrics $D(\cdot, \cdot)$. We also fix a confidence level $\delta\in(0,1)$ in the following. \\
\textbf{1. Total variation (TV) uncertainty set ($\cPhat^{\tv} $)}: We define $\cPhat^{\tv}  = \otimes \cPhat^{\tv}_{s,a}$, where $\cPhat^{\tv}_{s,a}$ is as in \eqref{eq:uncertainty-set} with the empirical estimator $\Phat^o_{s,a}$, the total variation distance $\dtv(P,\Phat^o_{s,a}) = (1/2) \|P - \Phat^{o}_{s,a} \|_{1}$, and radius
\begin{equation} 
\label{eq:tv-diameter}
    \rho_{s,a} = 1 \wedge \sqrt{\frac{\max\{\cardS,2\log(2|\cS||\cA|/\delta)\}}{N(s,a)}} \indic\{N(s,a)\geq 1\} .
\end{equation}
\textbf{2. Wasserstein uncertainty set ($\cPhat^{\wass}$)}: We define $\cPhat^{\wass}  = \otimes \cPhat^{\wass}_{s,a}$, where $\cPhat^{\wass}_{s,a}$ is as in \eqref{eq:uncertainty-set} with the empirical estimator $\Phat^o_{s,a}$, and with the Wasserstein distance 
$\dwass(P,\Phat^o_{s,a}) = \inf_{\nu\in \mathrm{m}(P,\Phat^o_{s,a})} \int \ell(x,y) \dd\nu(dx,dy)$, where the integration is over ${(x,y)\in\ScS}$, $\mathrm{m}(P,\Phat^o_{s,a})$ denotes all probability measures on $\ScS$ with marginals $P$ and $\Phat^o_{s,a}$, and $\ell(\cdot, \cdot)$ is the discrete metric, $\ell(s, s')=\indic\{s\neq s'\}$, and radius
\begin{align}
    \rho_{s,a} = 1 \wedge \sqrt{\frac{C|\cS|\log(|\cS||\cA|/\delta)}{N(s,a)}}\indic\{N(s,a)\geq 1)\},
\end{align}
where $C$ is a problem independent constant.\\
\textbf{3. Kullback-Leibler (KL) uncertainty set ($\cPhat^{\kl}$)}: We define $\cPhat^{\kl}  = \otimes \cPhat^{\kl}_{s,a}$, where $\cPhat^{\kl}_{s,a}$ is as in \eqref{eq:uncertainty-set} with the add-$L(=1)$ estimator $\Ptilde^o_{s,a}$, and with the KL distance
$\dkl(P,\Ptilde^o_{s,a})  = \sum_{s'} P(s') \log ({P(s')}/{\Ptilde^o_{s,a}(s')})$, and radius\begin{align}
\rho_{s,a}=\log(\cardS) \wedge \frac{C\cardS \log(\cardS^2\cardA/\delta)\log(N)}{N(s,a)}\indic\{N(s,a)\geq 1\}, 
\end{align}
where $C$ is a problem independent constant. \\
\textbf{4. Chi-square uncertainty set ($\cPhat^{\chisq}$)}: We define $\cPhat^{\chisq}  = \otimes \cPhat^{\chisq}_{s,a}$, where  $\cPhat^{\chisq}_{s,a}$ is  as in \eqref{eq:uncertainty-set} with the add-$L(=\log(1/\delta))$ estimator $\Ptilde^o_{s,a}$, and with the chi-square distance $\dchi(P,\Ptilde^o_{s,a}) =  \sum_{s'} {(P(s') - \Ptilde^o_{s,a}(s'))^2}/{\Ptilde^o_{s,a}(s')}$, and radius
\begin{align}
    \rho_{s,a}=(\cardS+1) \wedge \frac{C\cardS \log(\cardS^2\cardA/\delta)}{N(s,a)} \indic\{N(s,a)\geq 1\},
\end{align}
where $C$ is a problem independent constant.

We would like to emphasize that prior work on distributionally robust MDP/RL have shown that the empirical robust Bellman operator (\cref{eq:robust-bellman-eq}) can be evaluated in a computational tractable way for all the above four uncertainty sets \citep{iyengar2005robust, panaganti22a,ho2022robust,xu-panaganti-2023samplecomplexity,kumar2022efficient}. 
In view of these computational tractable methods, we only present
our DRQI algorithm using Q-iteration with the empirical robust Bellman operator (\cref{eq:robust-bellman-eq}), and is summarized in \cref{alg:MPQI-Algorithm}.
We now present the sample complexity of DRQI with TV uncertainty set, and a proof sketch. We obtain sample complexities of same order for all other uncertainty sets. We defer the corresponding  theorem statements and proofs to \cref{appen:drqi-proofs}. 
\begin{theorem}
    Let $\pi_K$ be the DRQI policy  after $K$ iterations under the TV uncertainty set $\cPhat^\tv$. If the total number of samples $N\geq N_\tv$, where
    \begin{align*}
    N_\tv = \cO\bigg( \frac{C_{\pi^*}\max\{\cardS^2,2\log(2\cardS^2\cardA/\delta)\}}{\epsilon^2(1-\gamma)^4} \bigg),
    \end{align*}
    then $\E_{s_0\sim d_0} [V^{\pi^*}(s_0) - \E_\cD [{V}^{\pi_K}(s_0)]] \leq \epsilon$ with probability at least $1-\delta$ and a sufficiently large $K$.
\end{theorem}
\begin{proof}[Proof Sketch]
Denoting $\cPhat^\tv$ simply as $\cPhat$, we first  write $V^{\pi^*}_{P^o}(s_0) - {V}^{\pi_K}_{P^o}(s_0) =  (  V^{\pi^*}_{P^o}(s_0) - V^{\pi_K}_{\cPhat}(s_0)) +  (  V^{\pi_K}_{\cPhat}(s_0) - {V}^{\pi_K}_{P^o}(s_0))$, where $V^{\pi_K}_{\cPhat} = \inf_{P \in \cPhat} ~V_{ P}^{\pi_{K}}$ is the robust value of policy $\pi_{K}$ corresponding to the uncertainty set $\cPhat$. In  \cref{prop:tv-high-prob-event-model-based} we show that, with the  $\rho_{s,a}$ as specified above, $P^o\in\cPhat^\tv$ with probability at least $1-\delta$. So, by definition of the robust value function, the second term $(  V^{\pi_K}_{\cPhat}(s_0) - {V}^{\pi_K}_{P^o}(s_0))$ is negative and we only need to bound the first term.   

To bound the first term, we decompose it as $(  V^{\pi^*}_{P^o}(s_0) - V^{\pi_K}_{\cPhat}(s_0)) = (V^{\pi^*}_{P^o}(s_0) - V^{\pihat^*}_{\cPhat}(s_0) )  +  ( V^{\pihat^*}_{\cPhat}(s_0) - V^{\pi_K}_{\cPhat}(s_0) )$, where  $\pihat^* = \argmax_{\pi}V^{\pi}_{\cPhat} $ is the optimal robust policy w.r.t. $\cPhat$. Then, due to the contraction propetry of the robust Bellman operator, $( V^{\pihat^*}_{\cPhat}(s_0) - V^{\pi_K}_{\cPhat}(s_0) )$ will converge to zero exponentially in $K$. 

Bounding $(V^{\pi^*}_{P^o}(s_0) - V^{\pihat^*}_{\cPhat}(s_0) )$ is more technical. The key idea is to first note that  $\dtv(P_{s,\pi^{*}(s)},P^o_{s,\pi^{*}(s)}) \leq 2 \rho_{s,a}$ for any $P \in \cPhat$, by \cref{prop:tv-high-prob-event-model-based}  and definition of $\cPhat$. Now, unrolling along the trajectory generated by $\pi^{*}$ on $P^{o}$ and using the  form of $\rho_{s,a}$, we can  get an upper bound in terms of $\E_{s\sim d^{\pi^*}} [1/\sqrt{N(s,\pi^*(s))}]$. We will then express $N(s,\pi^*(s))$ in terms of $N \mu(s,\pi^*(s))$ using \cref{lem:bound-on-binomial-inverse-moments},  and then use a change of measure argument to get the final bound in terms of single concentrability coefficient $C_{\pi^*}$.
\end{proof}

\section{Linear-MDP Distributionally Robust Q-Iteration (LM-DRQI) Algorithm}
\label{sec:lin-mpqi}

In this section, we  propose our LM-DRQI algorithm to solve offline RL problem in the linear MDP setting with large state space and finite actions, and  provide its sample complexity guarantees.

We now define the linear architecture called \emph{linear MDP} used in RL literature \citep{jin2020provably,jin2021pessimism,yin2022near} for handling large state space setting.
\begin{definition}[Linear MDP \citep{jin2020provably}]
\label{defn:linear-MDP}
We say an MDP $M = (\cS, \cA, r, P, \gamma)$ is a {linear MDP} with a known feature map   $\phi : \ScA \rightarrow \R^{d}$, if there exists $d$ unknown (signed) measures $\nu = (\nu_{1}(\cdot), \ldots, \nu_{d}(\cdot))$ over $\cS$ and an unknown vector $\theta \in \R^{d},$ such that for any $(s, a) \in \ScA,$ we have
\begin{align}
\label{eq:linear-MDP}
P_{s,a} = \langle \phi(s, a), \nu(\cdot)  \rangle,~~ r(s, a) = \langle \phi(s, a), \theta  \rangle.
\end{align}
\end{definition}
Similar to the tabular setting, here also we assume that the reward function (equivalently  $\theta$) is known, in order to focus on the key aspect of DRL formulation. We make the following assumptions. 
\begin{assumption}
\label{as:linear-MDP}
Let $M = (\cS, \cA, r, P^o, \gamma)$ be a {linear MDP} with a known feature map $\phi$ and unknown measure $\nu^o$. We assume that   $\phi_i(s, a) \geq 0$ for all $(s, a) \in \ScA$ and $i\in[d]$. We also assume that  $\Lambda =  \E_{s,a\sim\mu} [\phi(s,a)\phi(s,a)^\top] $ and $ \Sigma^{(i,j)}_{d^{\pi^*}}=\E_{s,a \sim d^{\pi^*}} [(\phi_i(s,a)\indic_i)(\phi_j(s,a)\indic_j)^\top]$ for all $i,j\in[d]$ are positive semi-definite matrices. 
\end{assumption}

We use the $d$-rectangularity uncertainty set construction which exploits the linear structure \citep{ma2022distributionally}. Instead of focusing on  the set of all models around $P^{o}$, we  consider only the set of linear models around $P^{o}$. This is  achieved indirectly by considering an uncertainty set around $\nu^{o}$ using  the integral probability metric (IPM) \citep{muller1997integral} and translating that to an uncertainty set around $P^{o}$ through the known feature vector $\phi$. More precisely, the $d$-rectangularity uncertainty set $\mathcal{P}$ is defined as 
\begin{align}
    &\cP =  \{ P: P_{s,a}(s') = \sum_{i\in[d]} \phi_i(s,a) \nu_i(s'), \nu_{i} \in \mathcal{M}_{i}, \forall i \in [d]  \},  \nn \\
    \label{eq:d-rectangularity-defn-2}
    &\mathcal{M}_{i} = \{ \nu_{i} :  D_{\mathrm{IPM}}(\nu_{i}, \nu^{o}_{i})  \leq \rho_{i} \},~\text{where},
\end{align}
$D_{\mathrm{IPM}}(p,q)=\sup_{V\in\cV}|\int_{s}(p(s)-q(s))V(s)\dd s|$, and 
$\cV=\{V(\cdot)=\max_{a} \phi^\top(\cdot,a) w :w\in\R^{d}, \|w\|_2\leq 1/(1-\gamma)\}$.

It is straight forward to show that the optimal robust value function is  linear w.r.t. $\phi$ under the  $d$-rectangularity uncertainty set. Moreover, we can also show that the robust Bellman operator (\cref{eq:robust-bellman-primal} can be written as 
\begin{align} 
\label{eq:linear-mdp-robust-bellman-eq}
  TQ(s,a)  = r(s,a) + \gamma \sum_{i\in[d]} \phi_i(s,a)  \min_{\nu_i \in \mathcal{M}_{i} }\E_{s'\sim \nu_i}(\max_{b} Q(s', b)).
\end{align}

We can get an empirical estimate $ \Phat^o$ of $P^{o}$ with ridge linear regression using the offline data \cite[Section 8.3]{agarwal2019reinforcement} as 
\begin{align}
    \Phat^o_{s,a}(s') &= \phi(s,a)^\top  \nuhat^{o}(s'), \text{ where} \\
     \nuhat^{o}(s') &= \frac{1}{N} \sum_{i=1}^N \Lambda_N^{-1} \phi(s_i,a_i) \indic\{s'=s_i'\}, \\
     \label{eq:lambda-N}
    \Lambda_N &= \frac{\lambda}{N} I +  \frac{1}{N} \sum_{i=1}^N\phi(s_i,a_i)\phi(s_i,a_i)^\top,
\end{align} 
and $\lambda$ is a constant.  We construct an estimate $\cMhat_i$ of $\mathcal{M}_{i}$ by replacing unknown $\nu^{o}_{i}$ with its estimate $\nuhat^{o}_{i}$. Similarly, we construct the empirical uncertainty set $\cPhat$ by replacing $\mathcal{M}_{i}$ by $\cMhat_i$. We fix the radius $\rho_{i}$ as 
\begin{align}
    \rho_i = \frac{c_1 \log(Nd/((1-\gamma)\delta)) }{1-\gamma} \sqrt{\frac{d}{N}}   \sqrt{\Lambda_N^{-1}(i,i)}.
\end{align}

We can now define the empirical robust Bellman operator $\widehat{T}$ exactly as in \cref{eq:linear-mdp-robust-bellman-eq}, but by replacing $\mathcal{M}_{i}$ by its estimate $\cMhat_i$. Our LM-DRQI algorithm then follows the same procedure as our DRQI algorithm using this $\widehat{T}$. We omit rewriting the algorithm procedure due to page limitation. 

\begin{figure*}[t]
    \centering 
    \begin{minipage}{.45\textwidth}
		\centering
		\includegraphics[width=\linewidth]{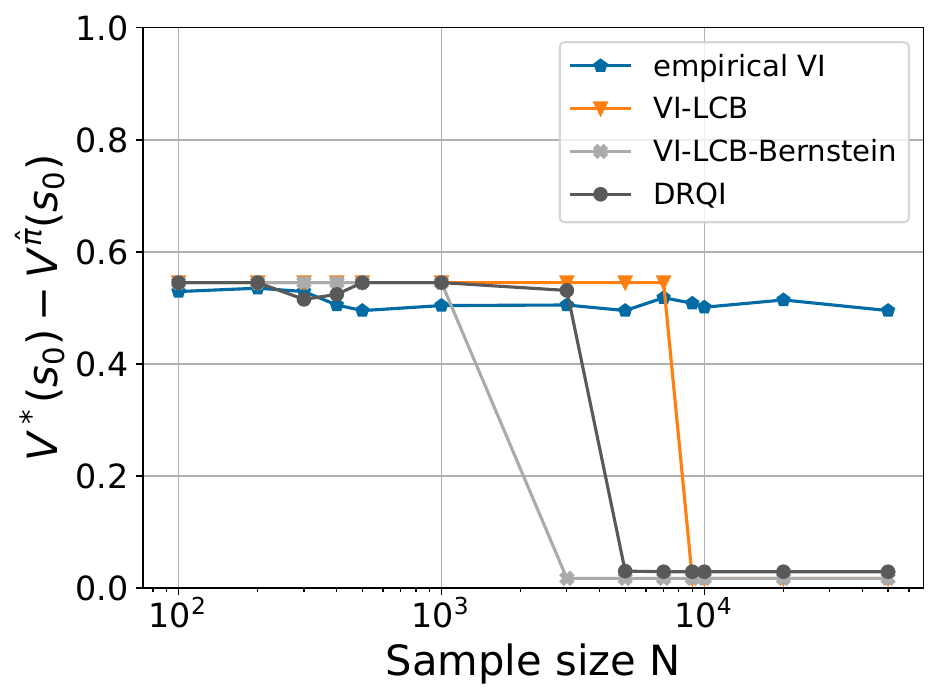}
		\caption{Convergence of DRQI algorithm under \textit{partial coverage} in \texttt{FrozenLake-v1}.}
		\label{fig:fh_partial}
	\end{minipage}\hspace{2em}
     \begin{minipage}{.45\textwidth}
        \centering
        \includegraphics[width=\linewidth]{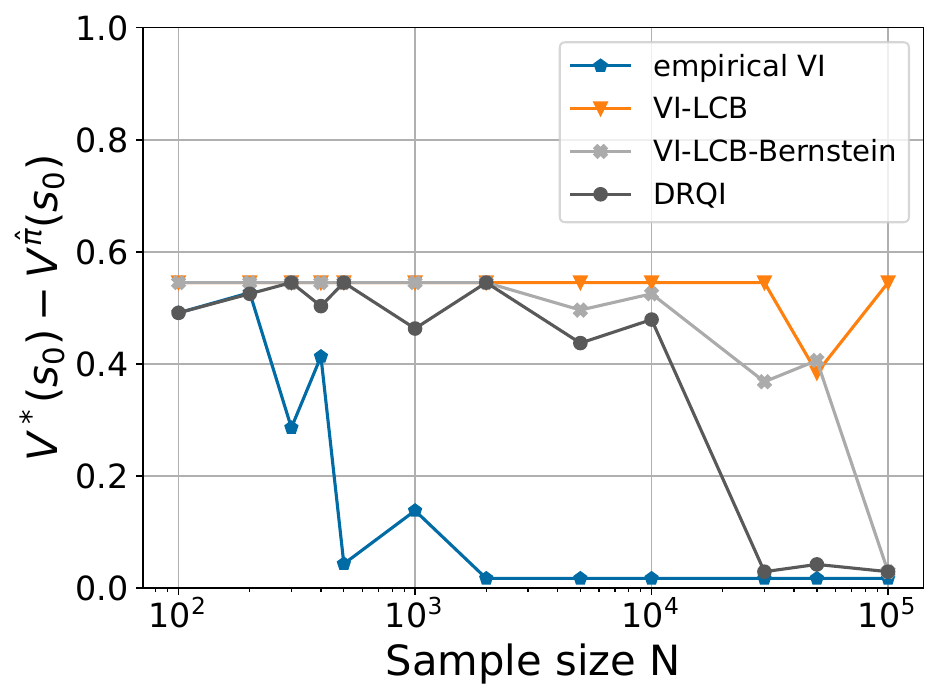}
        \caption{Convergence of DRQI algorithm under \textit{full coverage} in \texttt{FrozenLake-v1}.}
        \label{fig:fh_full}
    \end{minipage}
\end{figure*}

We make the following assumption that specifies coverage requirements to provide offline RL guarantees.
\begin{assumption}[Sufficient coverage assumption] \label{assum:linear-MDP-sufficient-coverage} For all $i\in[d]$, with probability $1-\delta$, it holds $\Lambda_N \geq (1/N)I + \drectsuffcov \cdot d \cdot \Sigma^{i}_{d^{\pi^*}_{P^o}}$, where $\Sigma^{i}_{d^{\pi^*}_{P^o}} = \E_{s,a \sim d^{\pi^*}_{P^o}} [(\phi_i(s,a)\indic_i)(\phi_i(s,a)\indic_i)^\top]$. 
\end{assumption}

The sufficient coverage assumption was originally used by  \citep{jin2021pessimism} for  showing that pessimism-based offline RL algorithms can learn optimal policy without assuming the uniform concentrability ($\rank(\Lambda)=d$ \citep{wang2021what} in linear MDPs). The \emph{sufficient coverage} assumption only requires that the  trajectory induced by the optimal policy $\pi^{*}$ is  covered by the offline data  sufficiently well. The assumption we use is from  \citet{ma2022distributionally}, which addressed the robust RL problem using offline data. This assumption  stipulate sufficient coverage in  each dimension $i \in [d]$. 
We now give the sample complexity of our LM-DRQI algorithm. 
\begin{theorem}
    Let $\pi_K$ be the LM-DRQI policy  after $K$ iterations. Let \cref{assum:linear-MDP-sufficient-coverage} hold. If the total number of samples $N\geq N_\mathrm{IPM}$, where
    \begin{align*}
        N_\mathrm{IPM} = \cOtilde( d\cdot\rank(\Sigma_{d^{\pi^*}_{P^o}})/({\drectsuffcov(1-\gamma)^4\epsilon^2}) ),
    \end{align*}
    then $\E_{s_0\sim d_0} [V^{\pi^*}(s_0) - \E_\cD [{V}^{\pi_K}(s_0)]] \leq \epsilon$ with probability at least $1-\delta$. 
    % \\
    % (ii) Let relative condition number $C^\dagger_{\pi^*,\phi}$ be small finite value. If the total number of samples $N\geq N^{(2)}_\mathrm{IPM}$, where
    % $N^{(2)}_\mathrm{IPM} = \cOtilde( {d C^\dagger_{\pi^*,\phi} \rank(\Lambda)^2 \log(d/\delta)}/({(1-\gamma)^4\epsilon^2}) ),$
    % then $\E_{s_0\sim d_0} [V^{\pi^*}(s_0) - \E_\cD [{V}^{\pi_K}(s_0)]] \leq \epsilon$ with probability at least $1-\delta$.
\end{theorem}
More detailed theorem statement and proofs are in \cref{appen:linear-mdp}. 
We remark this result is not directly comparable with that of \citet{uehara2021pessimistic} due to the disconnect between $\relativecondition$ and \cref{assum:linear-MDP-sufficient-coverage}. We include the LM-DRQI sample complexity guarantee for a variant of $\relativecondition$ in \cref{appen:linear-mdp} giving comparable results for a limited structure of linear MDPs.

\section{Experiments}
\label{sec:simu}

We evaluate the performance of our DRQI algorithm on the \texttt{FrozenLake-v1} environment ($\cardS=16$, $\cardA=4$) from OpenAI Gym \citep{brockman2016openai}. The goal is to cross a frozen lake without falling into holes. Since the frozen lake is slippery, rather than always going in the intended direction, the agent can slip into the other  directions. We implement DRQI algorithm with total variation uncertainty set using the CVXPY library \citep{diamond2016cvxpy} for the experiments. We submit our code in a Github repository: \href{https://github.com/zaiyan-x/DRQI}{\texttt{https://github.com/zaiyan-x/DRQI}}.

\paragraph{Offline Data Collection:} We evaluate the algorithms using  two kinds of offline datasets, \textit{full-coverage} and \textit{partial-coverage}. Full-coverage dataset is collected by using  a generative model where we collect equal number of next-state samples from every  $(s,a)$ pairs. The partial-coverage dataset is generated according to the behavior policy
\begin{equation*}
    \mu(a\mid s) = \frac{\indic{\{a=\pi^*(s)\}}}{2} + \frac{\indic{\{ a = \mathrm{unif}(\mathcal{A})\}}}{2},
\end{equation*}
where $\pi^{*}$ is the optimal policy for the \texttt{FrozenLake-v1} environment.  It is easy to check that the single-policy concentrability coefficient $C_{\pi^*}$ is bounded. Note that most of the $(s,a)$-pairs are un-sampled or under-sampled in the partial-coverage data set.

We compare our DRQI with three algorithms: (1) empirical value iteration (EVI) which essentially performs value iteration using the empirical model $\Phat^o$, (2) VI-LCB algorithm \citep{rashidinejad2022bridging}, a reward pessimism-based offline  RL algorithm, (3) VI-LCB-Bernstein algorithm   \citep{li2022settling}, a Bernstein type reward pessimism-based offline  RL algorithm. The performance metric is the value sub-optimality with respect to the optimal policy.

In the  partial data coverage setting (\cref{fig:fh_partial}), we see that the EVI algorithm does not converge even with $10^5$ samples, clearly showing the inability of standard dynamic programming approaches to obtain an approximately optimal policy in such settings. On the other hand, our DRQI algorithm learns the optimal policy  with roughly $4\times10^{3}$ samples. Moreover,  the performance of our DRQI algorithm is on par with the state-of-the-art VI-LCB and VI-LCB-Bernstein  offline RL algorithms (in fact performing better than VI-LCB but only slightly worse than VI-LCB-Bernstein).  Here, we also would like to note that both VI-LCB and VI-LCB-Bernstein algorithms require some hyperparameter tuning regarding the ``universal constants" that appear in their proofs of high-probability bounds. Our DRQI algorithm, on the other hand, does not require any hyperparameter tuning and use the $\rho_{s,a}$ exactly as defined in \cref{eq:tv-diameter}.

In the full data coverage setting (\cref{fig:fh_full}), EVI is able to find the optimal policy since the concentration of $\widehat{P}^o$ to the true model $P^o$ is straightforward. Our DRQI algorithm is also able to learn the optimal policy, albeit with more samples. Notably, our DRQI algorithm outperforms  the two LCB-style algorithms in this setting.

\section{Conclusion}
\label{sec:conclusion-mpqi}

In this work, we presented offline RL algoirthms for the tabular and linear MDP setting using the framework of DRL. We characterized the sample complexity of these algorithms only using the single policy concentrability assumption. We also demonstrated the superior performance our proposed algorithm
through simulation experiments. In the future, we plan to extend these results to general function approximation setting to handle large state-action space problem.

\section{Acknowledgments}

This work was supported in part by the National Science Foundation (NSF) grants NSF-CAREER-EPCN-2045783 and NSF ECCS 2038963. Any opinions, findings, and conclusions or recommendations expressed in this material are those of the authors and do not necessarily reflect the views of the sponsoring agencies.

\bibliography{References-Full} 

\clearpage

{\hfill \LARGE \bf \underline{\Coffeecup ~Supplementary Materials \Coffeecup} \hfill}
\appendix

\section{Useful Technical Results}
\begin{lemma}[Bound on binomial inverse moments \text{\citep[][Lemma 14]{rashidinejad2022bridging}}]\label{lem:bound-on-binomial-inverse-moments}
    Let $n\sim \Binomial(N,p)$. For any $k\geq0$, there exists a constant $c_k$ depending only on $k$ such that
    \begin{equation*}
        \E \bigg[ \frac{1}{(n\vee 1)^k}\bigg] \leq \frac{c_k}{(Np)^k},
    \end{equation*}
    where $c_k=1+k2^{k+1}+k^{k+1}+k\big(\frac{16(k+1)}{e}\big)^{k+1}$.
\end{lemma}

\begin{lemma}[Hoeffding's inequality \text{\cite[see Theorem 2.8]{boucheron2013concentration}}]\label{thm:hoeffding}
 Let $X_1,\dots,X_n$ be independent random variables such that $X_i$ takes its values in $[a_i,b_i]$ almost surely for all $i\leq n$. Let
 \begin{equation*}
     S=\sum_{i=1}^n(X_i-\expect{X_i}).
 \end{equation*}
 Then for every $t>0$,
 \begin{equation*}
     \prob{S\geq t}\leq\expnew{-\frac{2t^2}{\sum_{i=1}^n(b_i-a_i)^2}}.
 \end{equation*}
 Furthermore, if $X_1,\dots,X_n$ are a sequence of independent, identically distributed random variables with mean $\mu$. Let $\mean{X}_n = \frac{1}{n}\sum_{i=1}^n X_i$. Suppose that $X_i\in[a,b]$, $\forall i$. Then for all $t>0$
\begin{equation*}
     \prob{\abs{\mean{X}_n - \mu} \geq t} \leq 2\expnew{-\frac{2nt^2}{(b-a)^2}}.
\end{equation*}
\end{lemma}

The following lemmas characterize the sample complexity of learning discrete distributions when the accuracy is measured under four different distances, i.e., total variation, KL, chi-square, and Wasserstein.
\begin{lemma}[\text{\citealp[Theorem 1]{canonne2020short}}]\label{lem:tv-concen}
    Fix any $\delta\in (0,1]$. Let $\widehat{P}$ be the empirical distribution constructed from $N$ i.i.d. samples from an unknown distribution $P$ over a finite set $\{1,\dots,k\}$. Then if the number of samples
    \begin{equation*}
        N\geq \frac{\max\{k,2\log(2/\delta) \}}{\epsilon^2},
    \end{equation*}
    then $\dtv(P,\widehat{P}) \leq \epsilon$ with probability at least $1-\delta$. Moreover, this result is tight.
\end{lemma}

\begin{lemma}[\text{\citealp[Theorem 6.1]{bhattacharyya2021near}}]\label{lem:kl-concen}
    Fix any $\delta\in(0,1]$. Let $\Ptilde$ be the empirical add-$1$ estimator obtained from $N$ i.i.d. samples from an unknown distribution $P$ over a finite set $\{1,\dots,k\}$. There exists a universal constant $C$ such that, with probability at least $1-\delta$,
    \begin{equation*}
        \dkl(P,\Ptilde) \leq \frac{Ck\log(k/\delta)\log N}{N}.
    \end{equation*}
\end{lemma}

\begin{lemma}[\text{\citealp[Proposition 4.1]{arora2023near}}]\label{lem:chi-concen}
    Fix any $\delta\in(0,1]$ and let $L=\Theta(\log(1/\delta))$. Let $\Ptilde$ be the empirical add-$L$ estimator obtained from $N$ i.i.d. samples from an unknown distribution $P$ over a finite set $\{ 1,\dots,k\}$. There exists a universal constant $C$ such that, with probability at least $1-\delta$,
    \begin{equation*}
        \dchi(P,\Ptilde) \leq \frac{Ck\log(k/\delta)}{N}.
    \end{equation*}
\end{lemma}

\begin{lemma}[\text{\citealp[Corollary 5.2]{lei2020convergence}}]\label{lem:wass-concen}
    Let $p\in\cP(\R^d)$ be a distribution such that $b = \E_{X\sim p}[\exp(a\|X\|_2)] < \infty$ for some $a > 0$. Fix $\delta\in(0,1)$. Denote the empirical distribution from $N$ samples of $p$ as $\widehat{p}$. Then there exists some constant $c_1$ only depending on $a, b$ such that $\dwass(p,\widehat{p})\leq \sqrt{c_1 d\log(1/\delta)/ N}$ holds at least with probability $1-\delta$.
\end{lemma}

Here we mention a uniform concentration result from \citet{agarwal2019reinforcement} corresponding to linear MDP transition model $P^o$. From \cref{sec:lin-mpqi}, recall $\Lambda_N$ and the model estimate of $P^o$ denoted by  $\Phat^o$. We note that $\cO$ notation in this result only removes dependence on universal constants.
\begin{lemma}[Linear MDP Uniform Concentration Bound \text{\citep[][Lemma 8.7]{agarwal2019reinforcement}}] \label{lem:linear-MDP-uniform-concentrability}
    Fix $\delta\in(0,1)$ and let $\lambda=1$. Consider $\cV=\{V(\cdot)=\max_{a} \phi^\top(\cdot,a) w :w\in\R^{d}, \|w\|_2\leq 1/(1-\gamma)\}$. We have (1) $\| \sum_{t=1}^N\phi(s_t,a_t) \epsilon_t^\top V  \|_{\Lambda_N^{-1}} \leq \cO( {\sqrt{dN}\log(N/((1-\gamma)\delta))}/{ 1-\gamma} )$  with probability at least $1-\delta$ for any $V\in\cV$ uniformly, and it also holds
    (2) $\sup_{V\in\cV}|\int_{\cS} (P^o_{s,a} - \Phat^o_{s,a}) V(ds') | \leq \| \phi(s,a) \|_{\Lambda_N^{-1}} \cdot $\\$\cO\left( {\sqrt{d}\log(N/\delta)}/{ ((1-\gamma) \sqrt{N})} \right)$ with probability at least $1-\delta$ for any $s,a$ and  for any $V\in\cV$ uniformly. 
\end{lemma}

Here is a useful result from \citep[Theorem 21]{chang2021mitigating}. 
\begin{lemma}\label{lem:data-generation-inequality}
    Let $\lambda=1$ and $c>0$ be some universal constant.
    For all $s,a$ simultaneously, with probability at least $1-\delta$ we have $\E_{s,a\sim \mu} [  \phi(s,a)^\top \Lambda_N^{-1} \phi(s,a)  ]  \leq c^2 \cdot  \rank(\Lambda) (\rank(\Lambda) + \log(c/\delta))$
    where $\Lambda_N = \frac{\lambda}{N} I +  \frac{1}{N} \sum_{t=1}^N\phi(s_t,a_t)\phi(s_t,a_t)^\top$, $\Lambda =  \E_{s,a\sim\mu}\phi(s,a)\phi(s,a)^\top$.
\end{lemma}

\section{Proofs of Distributionally Robust Q-Iteration (DRQI)} \label{appen:drqi-proofs}

We first make the observation that the true model $P^o$ lies in the uncertainty set $\cPhat$ with high probability. Intuitively, the empirical estimator $\Phat^o$ of $P^o$ are statistically closer which is dependent on the number of samples. We first make this observation and intuition formal in the proposition below for the TV uncertainty set.
\begin{proposition}\label{prop:tv-high-prob-event-model-based}
We have $P^o\in\cPhat^\tv$ with probability at least $1-\delta$.
\end{proposition}
\begin{proof}
We start with the fact that $\dtv(p,q) \leq  1$ for any distributions $p,q$. For the case $N(s,a)< 1$, i.e., $N(s,a)= 0$, it is trivial that $P^o_{s,a}\in\cPhat^\tv_{s,a}$, almost surely, since $\cPhat^\tv_{s,a}=\Delta(\cS)$.
 
From \cref{lem:tv-concen}, we have $\dtv(P^o_{s,a},\Phat^o_{s,a}) \leq \sqrt{{\max\{\cardS,2\log(2/\delta)\}}/{N(s,a)}}$ for any $s,a$ pair with probability at least $1-\delta/(|\cS||\cA|)$. Thus $\otimes_{s,a}P^o_{s,a}\in\otimes_{s,a}\cPhat^\tv_{s,a}$ holds with probability at least $1-\delta$.
\end{proof}

We now provide a similar guarantee like \cref{prop:tv-high-prob-event-model-based} for the Wasserstein uncertainty set.
\begin{proposition}\label{prop:wass-high-prob-event-model-based}
We have $P^o\in\cPhat^\wass$ with probability at least $1-\delta$.
\end{proposition}
\begin{proof}
From \cref{prop:tv-high-prob-event-model-based} and \citet[Theorem 6.15]{villani2009optimal}, it follows that $\dwass(p,U) \leq  1$ for any distribution $p$ and uniform distribution $U$, i.e., $U(s)=1/\cardS$ for all $s\in\cS$. For the case $N(s,a)< 1$, i.e., $N(s,a)= 0$, it now follows that $P^o_{s,a}\in\cPhat^\wass_{s,a}$, almost surely, since $\Phat^o_{s,a}=1/\cardS$.

From \cref{lem:wass-concen}, we have $\dwass(P^o_{s,a},\Phat^o_{s,a}) \leq \sqrt{C_{s,a}\cardS \log(1/\delta)/{N(s,a)}}$ for any $s,a$ pair with probability at least $1-\delta/(|\cS||\cA|)$, where $C_{s,a}>0$ is some universal constant depending only on the distribution $P^o_{s,a}$. By choosing uniform $C$ over all $C_{s,a}$ and $P^o_{s,a}$, $\otimes_{s,a}P^o_{s,a}\in\otimes_{s,a}\cPhat^\wass_{s,a}$ holds with probability at least $1-\delta$.
\end{proof}

We now provide a similar guarantee for the KL uncertainty set.
\begin{proposition}\label{prop:kl-high-prob-event-model-based}
We have $P^o\in\cPhat^\kl$ with probability at least $1-\delta$.
\end{proposition}
\begin{proof}
We start with the fact that $\dkl(p,U) \leq  \log(\cardS)$ for any distribution $p$ and uniform distribution $U$, i.e., $U(s)=1/\cardS$ for all $s\in\cS$. For the case $N(s,a)< 1$, i.e., $N(s,a)= 0$, it now follows that $P^o_{s,a}\in\cPhat^\kl_{s,a}$, almost surely, since $\Ptilde^o_{s,a}=1/\cardS$.

From \cref{lem:kl-concen}, we have $\dkl(P^o_{s,a},\Ptilde^o_{s,a}) \leq C\cardS \log(\cardS/\delta)\log(N(s,a))/{N(s,a)}$ for any $s,a$ pair with probability at least $1-\delta/(|\cS||\cA|)$, where $C>0$ is some universal constant. We also know that $\log(N(s,a))\leq \log(N)$. Thus $\otimes_{s,a}P^o_{s,a}\in\otimes_{s,a}\cPhat^\kl_{s,a}$ holds with probability at least $1-\delta$.
\end{proof}

We now provide a similar guarantee for the chi-square uncertainty set.
\begin{proposition}\label{prop:chi-square-high-prob-event-model-based}
We have $P^o\in\cPhat^\chisq$ with probability at least $1-\delta$.
\end{proposition}
\begin{proof}
We start with the fact that $\dchi(p,U) \leq  \cardS+1$ for any distribution $p$ and uniform distribution $U$, i.e., $U(s)=1/\cardS$ for all $s\in\cS$. For the case $N(s,a)< 1$, i.e., $N(s,a)= 0$, it now follows that $P^o_{s,a}\in\cPhat^\chisq_{s,a}$, almost surely, since $\Ptilde^o_{s,a}=1/\cardS$.

From \cref{lem:chi-concen}, we have $\dchi(P^o_{s,a},\Ptilde^o_{s,a}) \leq C\cardS \log(\cardS/\delta)/{N(s,a)}$ for any $s,a$ pair with probability at least $1-\delta/(|\cS||\cA|)$, where $C>0$ is some universal constant. Thus $\otimes_{s,a}P^o_{s,a}\in\otimes_{s,a}\cPhat^\chisq_{s,a}$ holds with probability at least $1-\delta$.
\end{proof}

We are now ready to present our main results of \cref{sec:mpqi}. With the above result (\cref{prop:tv-high-prob-event-model-based}), we now provide the offline RL suboptimality guarantee below for the TV uncertainty set.
\begin{theorem}\label{thm:tv-tabular-offlineRL-subopt} Let $\pi_K$ be the DRQI policy  after $K$ iterations under the TV uncertainty set $\cPhat^\tv$. With probability at least $1-\delta$ it holds that
    \begin{align*}
    \E_{s_0\sim d_0} [V^{\pi^*}(s_0) - \E_\cD [{V}^{\pi_K}(s_0)]] \leq  \frac{64\gamma \sqrt{C_{\pi^*}|\cS|}}{(1-\gamma)^2}  \sqrt{\frac{\max\{\cardS,2\log(2|\cS||\cA|/\delta)\}}{N }} + \frac{2\gamma^{K+1}}{(1-\gamma)^2}.
\end{align*}
\end{theorem}
\begin{proof}
We first make important definitions that will be useful for our analyses. We denote the value function of policy $\pi$ for the transition dynamics model $P$ as $V^\pi_P$. We now denote the robust value function \cite{panaganti22a,xu-panaganti-2023samplecomplexity,panaganti-rfqi} for uncertainty set $\cPhat^\tv$ as 
$V^{\pi}_{\cPhat} = \min_{P\in\cPhat} V^\pi_P$ and its optimal robust policy as $\pihat^* = \argmax_{\pi}V^{\pi}_{\cPhat} $. We note that for the sake of notational simplicity we drop the superscript $\tv$ going forward, that is, we denote $\cPhat^\tv$ simply as $\cPhat$. We let $Q^{\pi}_{\cPhat}$ be its corresponding robust Q-function. 
From robust RL \cite{panaganti22a,xu-panaganti-2023samplecomplexity,panaganti-rfqi} we can write the following robust Bellman equation: $Q^{\pi}_{\cPhat}(s,a)=r(s,a) + \gamma \min_{P_{s,a}\in\cPhat_{s,a}}\E_{s'\sim P_{s,a}}(V^{\pi}_{\cPhat}(s'))$. To make it notationally easy, we write $V^{\pi^*}$ ($d^\pi$) as $V^{\pi^*}_{P^o}$ ($d^\pi_{P^o}$) making the dependence on the model $P^o$ explicit.

We now start analyzing offline RL suboptimality as:
\begin{align}
	\E_{s_0\sim d_0} [V^{\pi^*}_{P^o}(s_0) - {V}^{\pi_K}_{P^o}(s_0)] &= \E_{s_0\sim d_0} [V^{\pi^*}_{P^o}(s_0) - V^{\pi_K}_{\cPhat}(s_0) + V^{\pi_K}_{\cPhat}(s_0) - {V}^{\pi_K}_{P^o}(s_0)]\nn\\
    &\stackrel{(a)}{\leq} \E_{s_0\sim d_0} [V^{\pi^*}_{P^o}(s_0) - V^{\pi_K}_{\cPhat}(s_0) ]\nn\\
    &= \E_{s_0\sim d_0} [V^{\pi^*}_{P^o}(s_0) - V^{\pihat^*}_{\cPhat}(s_0) + V^{\pihat^*}_{\cPhat}(s_0) - V^{\pi_K}_{\cPhat}(s_0) ]\nn\\
    &\leq \E_{s_0\sim d_0} [V^{\pi^*}_{P^o}(s_0) - V^{\pihat^*}_{\cPhat}(s_0)] + \norm{V^{\pihat^*}_{\cPhat} - V^{\pi_K}_{\cPhat} }_\infty \nn\\
    &\stackrel{(b)}{\leq}  \E_{s_0\sim d_0} [V^{\pi^*}_{P^o}(s_0) - V^{\pihat^*}_{\cPhat}(s_0)] + \frac{2\gamma^{K+1}}{(1-\gamma)^2}, \label{eq:mpqi-part-0}
\end{align}
where $(a)$ follows from \cref{prop:tv-high-prob-event-model-based} and definition of robust value function $V^{\pi_K}_{\cPhat}(s_0)$ and
$(b)$ follows from robust amplification lemma \cite[Lemma 10, eq.(28)]{panaganti22a}. For the rest of the analysis, we focus on analyzing $\E_{s_0\sim d_0} [V^{\pi^*}_{P^o}(s_0) - V^{\pihat^*}_{\cPhat}(s_0)]$.

Observe that,
\begin{align}
    &\E_{s_0\sim d_0} [V^{\pi^*}_{P^o}(s_0) - V^{\pihat^*}_{\cPhat}(s_0)] = \E_{s_0\sim d_0} [Q^{\pi^*}_{P^o}(s_0,\pi^*(s_0)) - Q^{\pihat^*}_{\cPhat}(s_0,\pihat^*(s_0))] \nn\\
    &\stackrel{(c)}{\leq} \E_{s_0\sim d_0} [Q^{\pi^*}_{P^o}(s_0,\pi^*(s_0)) - Q^{\pihat^*}_{\cPhat}(s_0,\pi^*(s_0))]\nn\\
    &\stackrel{(d)}{=} \E_{s_0\sim d_0} [r(s_0,\pi^*(s_0)) + \gamma \E_{s'\sim P^o_{s_0,\pi^*(s_0)}}(V^{\pi^*}_{P^o}(s')) \nn\\
    &\hspace{1cm}-r(s_0,\pi^*(s_0)) - \gamma \min_{P_{s_0,\pi^*(s_0)}\in\cPhat_{s_0,\pi^*(s_0)}}\E_{s'\sim P_{s_0,\pi^*(s_0)}}(V^{\pihat^*}_{\cPhat}(s')) ]\nn\\
    &= \E_{s_0\sim d_0} [ \gamma \E_{s'\sim P^o_{s_0,\pi^*(s_0)}}(V^{\pi^*}_{P^o}(s')) - \gamma \E_{s'\sim P^o_{s_0,\pi^*(s_0)}}(V^{\pihat^*}_{\cPhat}(s'))]  \nn\\
    &\hspace{1cm}+\E_{s_0\sim d_0} [\gamma \E_{s'\sim P^o_{s_0,\pi^*(s_0)}}(V^{\pihat^*}_{\cPhat}(s')) - \gamma \min_{P_{s_0,\pi^*(s_0)}\in\cPhat_{s_0,\pi^*(s_0)}}\E_{s'\sim P_{s_0,\pi^*(s_0)}}(V^{\pihat^*}_{\cPhat}(s')) ]\nn\\
    &= \E_{s_0\sim d_0} [ 
    \gamma \E_{s'\sim P^o_{s_0,\pi^*(s_0)}}(V^{\pi^*}_{P^o}(s')-V^{\pihat^*}_{\cPhat}(s')) ]  \nn\\
    &\hspace{1cm}+\underbrace{ \E_{s_0\sim d_0} [\gamma \E_{s'\sim P^o_{s_0,\pi^*(s_0)}}(V^{\pihat^*}_{\cPhat}(s'))  - \gamma \min_{P_{s_0,\pi^*(s_0)}\in\cPhat_{s_0,\pi^*(s_0)}}\E_{s'\sim P_{s_0,\pi^*(s_0)}}(V^{\pihat^*}_{\cPhat}(s'))]}_{(I)} , \label{eq:mpqi-part-1}
\end{align}
where $(c)$ follows since $\pihat^* $ is optimal robust policy of $V^{\pi}_{\cPhat}$ and $(d)$ follows from classical and robust Bellman equations.

Analyzing $(I)$ in \cref{eq:mpqi-part-1} but for any $P_{s_0,\pi^*(s_0)}\in\cPhat_{s_0,\pi^*(s_0)}$ gives us:
\begin{align}
    &\E_{s_0\sim d_0} [\gamma \E_{s'\sim P^o_{s_0,\pi^*(s_0)}}(V^{\pihat^*}_{\cPhat}(s'))  - \gamma \E_{s'\sim P_{s_0,\pi^*(s_0)}}(V^{\pihat^*}_{\cPhat}(s'))] \label{eq:mpqi-analyze-I}\\
    &=\E_{s_0\sim d_0} [ \gamma \E_{s'\sim P^o_{s_0,\pi^*(s_0)}}(V^{\pihat^*}_{\cPhat}(s')) - \gamma \E_{s'\sim \Phat^o_{s_0,\pi^*(s_0)}}(V^{\pihat^*}_{\cPhat}(s')) \nn\\
    &\hspace{2cm} +\gamma \E_{s'\sim \Phat^o_{s_0,\pi^*(s_0)}}(V^{\pihat^*}_{\cPhat}(s'))  - \gamma \E_{s'\sim P_{s_0,\pi^*(s_0)}}(V^{\pihat^*}_{\cPhat}(s'))] \nn\\
    &\stackrel{(g)}{\leq}   \frac{2\gamma}{1-\gamma} \E_{s_0\sim d_0} \bigg[ \min \bigg\{1, \sqrt{\frac{\max\{\cardS,2\log(2|\cS||\cA|/\delta)\}}{N(s_0,\pi^*(s_0))}} \;\indic\{N(s_0,\pi^*(s_0))\geq 1 \} \bigg\} \bigg] \nn\\
    &\hspace{2cm} +\gamma \E_{s_0\sim d_0} [\E_{s'\sim \Phat^o_{s_0,\pi^*(s_0)}}(V^{\pihat^*}_{\cPhat}(s'))  -  \E_{s'\sim P_{s_0,\pi^*(s_0)}}(V^{\pihat^*}_{\cPhat}(s'))]\nn \\
    &\stackrel{(h)}{\leq}   \frac{4\gamma}{1-\gamma} \E_{s_0\sim d_0} \bigg[ \sqrt{\frac{\max\{\cardS,2\log(2|\cS||\cA|/\delta)\}}{N(s_0,\pi^*(s_0))}} \;\indic\{N(s_0,\pi^*(s_0))\geq 1 \}  \bigg], \label{eq:mpqi-part-2}
\end{align}
where $(g)$, holds with probability at least $1-\delta$, follows from Hölder's inequality and by \cref{prop:tv-high-prob-event-model-based}, and $(h)$ by Hölder's inequality and the definition of uncertainty set $\cPhat$.

Substituting \cref{eq:mpqi-part-2} back in  \cref{eq:mpqi-part-1}, we get the following recursion
\begin{align*}
    \E_{s_0\sim d_0} [V^{\pi^*}_{P^o}(s_0)& - V^{\pihat^*}_{\cPhat}(s_0)] \leq \gamma \E_{s_0\sim d_0} [ \E_{s'\sim P^o_{s_0,\pi^*(s_0)}}(V^{\pi^*}_{P^o}(s')-V^{\pihat^*}_{\cPhat}(s')) ] \nn\\
    &\hspace{0.5cm}+ \frac{4\gamma}{1-\gamma} \E_{s_0\sim d_0} \bigg[ \sqrt{\frac{\max\{\cardS,2\log(2|\cS||\cA|/\delta)\}}{N(s_0,\pi^*(s_0))}} \;\indic\{N(s_0,\pi^*(s_0))\geq 1 \}  \bigg] \nn\\
    &=\gamma \E_{s_1\sim d^{\pi^*}_{P^o,1}} [V^{\pi^*}_{P^o}(s_1)-V^{\pihat^*}_{\cPhat}(s_1) ] \nn\\
    &\hspace{0.5cm}+ \frac{4\gamma}{1-\gamma} \E_{s_0\sim d_0} \bigg[ \sqrt{\frac{\max\{\cardS,2\log(2|\cS||\cA|/\delta)\}}{N(s_0,\pi^*(s_0))}} \;\indic\{N(s_0,\pi^*(s_0))\geq 1 \}  \bigg] \nn\\
    &\leq \gamma^2 \E_{s_2\sim d^{\pi^*}_{P^o,2}} [V^{\pi^*}_{P^o}(s_2)-V^{\pihat^*}_{\cPhat}(s_2) ] \nn\\
    &\hspace{0.5cm}+ \gamma\frac{4\gamma}{1-\gamma} \E_{s_1\sim d^{\pi^*}_{P^o,1}} \bigg[ \sqrt{\frac{\max\{\cardS,2\log(2|\cS||\cA|/\delta)\}}{N(s_1,\pi^*(s_1))}} \;\indic\{N(s_1,\pi^*(s_1))\geq 1 \}  \bigg] \nn\\
    &\hspace{0.5cm}+\frac{4\gamma}{1-\gamma} \E_{s_0\sim d_0} \bigg[ \sqrt{\frac{\max\{\cardS,2\log(2|\cS||\cA|/\delta)\}}{N(s_0,\pi^*(s_0))}} \;\indic\{N(s_0,\pi^*(s_0))\geq 1 \}  \bigg] \nn\\
    &\leq\frac{4\gamma}{1-\gamma}  \sum_{t=0}^\infty \gamma^t \E_{s_t\sim d^{\pi^*}_{P^o,t}} \bigg[ \sqrt{\frac{\max\{\cardS,2\log(2|\cS||\cA|/\delta)\}}{N(s_t,\pi^*(s_t))}} \;\indic\{N(s_t,\pi^*(s_t))\geq 1 \}  \bigg] \nn\\
    &= \frac{4\gamma}{(1-\gamma)^2}  \E_{s\sim d^{\pi^*}_{P^o}} \bigg[ \sqrt{\frac{\max\{\cardS,2\log(2|\cS||\cA|/\delta)\}}{N(s,\pi^*(s))}} \;\indic\{N(s,\pi^*(s))\geq 1 \}  \bigg] ,
\end{align*}
where last equality follows by the definition of state-distribution $d^{\pi^*}_{P^o} = (1-\gamma) \sum_{t=0}^\infty \gamma^t d^{\pi^*}_{P^o,t}$. Now, putting this back in \cref{eq:mpqi-part-0}, we see that the offline RL guarantee becomes:
\begin{align}
    \E_\cD &[\E_{s_0\sim d_0} [V^{\pi^*}_{P^o}(s_0) - {V}^{\pi_K}_{P^o}(s_0)]]   \nn\\
    &\leq \frac{2\gamma^{K+1}}{(1-\gamma)^2}+\frac{4\gamma}{(1-\gamma)^2}  \E_{s\sim d^{\pi^*}_{P^o}} \E_\cD \bigg[ \sqrt{\frac{\max\{\cardS,2\log(2|\cS||\cA|/\delta)\}}{N(s,\pi^*(s))}} \;\indic\{N(s,\pi^*(s))\geq 1 \}  \bigg] \nn \\
    &\leq \frac{2\gamma^{K+1}}{(1-\gamma)^2}+\frac{4\gamma}{(1-\gamma)^2}  \E_{s\sim d^{\pi^*}_{P^o}} \E_\cD \bigg[ \sqrt{\frac{\max\{\cardS,2\log(2|\cS||\cA|/\delta)\}}{N(s,\pi^*(s)) \vee 1}} \bigg] \nn \\
    &\stackrel{(i)}{\leq}\frac{2\gamma^{K+1}}{(1-\gamma)^2} +  \frac{4\gamma}{(1-\gamma)^2}  \E_{s\sim d^{\pi^*}_{P^o}} \bigg[\sqrt{\max\{\cardS,2\log(2|\cS||\cA|/\delta)\}}\frac{16}{\sqrt{N \mu(s,\pi^*(s))}}\bigg] \nn\\
    &\stackrel{(j)}{\leq} \frac{2\gamma^{K+1}}{(1-\gamma)^2} +  \frac{64\gamma}{(1-\gamma)^2}  \E_{s\sim d^{\pi^*}_{P^o}} \bigg[\sqrt{\frac{ C_{\pi^*}\max\{\cardS,2\log(2|\cS||\cA|/\delta)\}}{N d^{\pi^*}_{P^o}(s,\pi^*(s)) }}\bigg]\nn\\
    &= \frac{2\gamma^{K+1}}{(1-\gamma)^2} + \frac{64\gamma \sqrt{C_{\pi^*}}}{(1-\gamma)^2}  \sqrt{\frac{\max\{\cardS,2\log(2|\cS||\cA|/\delta)\}}{N }} \sum_{s} \sqrt{d^{\pi^*}_{P^o}(s,\pi^*(s))}\nn \\
    &\stackrel{(k)}{\leq} \frac{2\gamma^{K+1}}{(1-\gamma)^2} + \frac{64\gamma \sqrt{ C_{\pi^*}}}{(1-\gamma)^2}  \sqrt{\frac{\max\{\cardS,2\log(2|\cS||\cA|/\delta)\}}{N }}\sqrt{|\cS|}. \label{eq:tv-thm-final-eq}
\end{align}
Recall that $(s_i,a_i)$-pairs in $\cD$ are i.i.d. and follow the data generating policy $\mu$. That is, for any $(s,a)$, $N(s,a)$ follows $\Binomial(N,\mu(s,a))$. Then $(i)$ follows from \cref{lem:bound-on-binomial-inverse-moments} with $k=1/2$. We note here that this technique of bridging two visitation distributions, $\mu$ and $d^{\pi^*}_{P^o}$, is critical and original in our paper. We have $(j)$ by recalling the definition of single-policy concentrability with comparator policy $\pi^*$, that is, \[ C_{\pi^*} = \max_{s,a} \frac{ d^{\pi^*}_{P^o}(s,a) }{\mu(s,a)}. \] $(k)$ is due to Cauchy-Schwarz inequality and by recognizing $d^{\pi^*}_{P^o}(\cdot,\pi^*(\cdot))$ as a probability distribution. This completes the proof of this main theorem.
\end{proof}

We now provide a similar offline RL suboptimality guarantee below for the Wasserstein uncertainty set using \cref{prop:wass-high-prob-event-model-based}.
\begin{theorem}\label{thm:wass-tabular-offlineRL-subopt} Let $\pi_K$ be the DRQI policy  after $K$ iterations under the Wasserstein uncertainty set $\cPhat^\wass$. With probability at least $1-\delta$ it holds that
    \begin{align*}
    \E_{s_0\sim d_0} [V^{\pi^*}(s_0) - \E_\cD [{V}^{\pi_K}(s_0)]] \leq  \frac{64\gamma \sqrt{C_{\pi^*}}}{(1-\gamma)^2}  \sqrt{\frac{C\cardS^2 \log(\cardS\cardA/\delta)}{N }} + \frac{2\gamma^{K+1}}{(1-\gamma)^2}.
\end{align*}
\end{theorem}
\begin{proof}
The proof follows exactly as in the proof of \cref{thm:tv-tabular-offlineRL-subopt}. We replace the dependence on \cref{prop:tv-high-prob-event-model-based} with \cref{prop:wass-high-prob-event-model-based}. 
    We then only have to take care of step $(g)$ in \cref{eq:mpqi-part-2}. We start from analyzing $(I)$ as in \cref{eq:mpqi-analyze-I}:
\begin{align*}
    &\E_{s_0\sim d_0} [\gamma \E_{s'\sim P^o_{s_0,\pi^*(s_0)}}(V^{\pihat^*}_{\cPhat}(s'))  - \gamma \E_{s'\sim P_{s_0,\pi^*(s_0)}}(V^{\pihat^*}_{\cPhat}(s'))] \nn\\
    &= \E_{s_0\sim d_0} [ \gamma \E_{s'\sim P^o_{s_0,\pi^*(s_0)}}(V^{\pihat^*}_{\cPhat}(s')) - \gamma \E_{s'\sim \Phat^o_{s_0,\pi^*(s_0)}}(V^{\pihat^*}_{\cPhat}(s')) \nn\\
    &\hspace{2cm} +\gamma \E_{s'\sim \Phat^o_{s_0,\pi^*(s_0)}}(V^{\pihat^*}_{\cPhat}(s'))  - \gamma \E_{s'\sim P_{s_0,\pi^*(s_0)}}(V^{\pihat^*}_{\cPhat}(s'))] \\
    &\stackrel{(a)}{\leq} \E_{s_0\sim d_0} \bigg[ \frac{2\gamma}{1-\gamma} \dwass( P^o_{s_0,\pi^*(s_0)}, \Phat^o_{s_0,\pi^*(s_0)} )   \nn\\
    &\hspace{2cm} +\gamma \E_{s'\sim \Phat^o_{s_0,\pi^*(s_0)}}(V^{\pihat^*}_{\cPhat}(s'))  - \gamma \E_{s'\sim P_{s_0,\pi^*(s_0)}}(V^{\pihat^*}_{\cPhat}(s')) \bigg] \\
    &\stackrel{(b)}{\leq}   \frac{2\gamma}{1-\gamma} 
    \E_{s_0\sim d_0}\bigg[ \sqrt{\frac{C\cardS \log(\cardS\cardA/\delta)}{N(s_0,\pi^*(s_0))}}\indic\{N(s_0,\pi^*(s_0))\geq 1\} \bigg] \nn\\
    &\hspace{2cm} +\gamma \E_{s_0\sim d_0} [\E_{s'\sim \Phat^o_{s_0,\pi^*(s_0)}}(V^{\pihat^*}_{\cPhat}(s'))  - \gamma \E_{s'\sim P_{s_0,\pi^*(s_0)}}(V^{\pihat^*}_{\cPhat}(s'))] \nn\\
    &\stackrel{(c)}{\leq}   \frac{4\gamma}{1-\gamma} \E_{s_0\sim d_0}\bigg[
    \sqrt{\frac{C\cardS \log(\cardS\cardA/\delta)}{N(s_0,\pi^*(s_0))}}\indic\{N(s_0,\pi^*(s_0))\geq 1\} \bigg],
    \end{align*}
    where $(a)$ follows by applying the  Kantorovich-Rubinstein theorem \citep[Theorem 11.8.2]{dudley2002real} and noting the fact that the value functions are $2/(1-\gamma)$-Lipschitz in their state dimension under the discrete metric $\ell(\cdot,\cdot)$ since $\norm{V^{\pihat^*}_{\cPhat}}_\infty \leq 1/(1-\gamma)$,
    $(b)$ holds with probability at least $1-\delta$ by \cref{prop:wass-high-prob-event-model-based}, and $(c)$ is again by the Kantorovich-Rubinstein theorem and the definition of uncertainty set $\cPhat$. Now combining and analyzing the rest of the steps as in the proof of \cref{thm:tv-tabular-offlineRL-subopt} completes the proof.
\end{proof}

We now provide a similar offline RL suboptimality guarantee below for the KL uncertainty set using \cref{prop:kl-high-prob-event-model-based}.
\begin{theorem}\label{thm:kl-tabular-offlineRL-subopt} Let $\pi_K$ be the DRQI policy  after $K$ iterations under the KL uncertainty set $\cPhat^\kl$ (under add-1 estimator). With probability at least $1-\delta$ it holds that
    \begin{align*}
    \E_{s_0\sim d_0} [V^{\pi^*}(s_0) - \E_\cD [{V}^{\pi_K}(s_0)]] \leq  \frac{64\gamma \sqrt{C_{\pi^*}}}{(1-\gamma)^2}  \sqrt{\frac{C\cardS^2 \log(\cardS^2\cardA/\delta)\log(N)}{N }} + \frac{2\gamma^{K+1}}{(1-\gamma)^2}.
\end{align*}
\end{theorem}
\begin{proof}
    The proof again follows exactly as in the proof of \cref{thm:tv-tabular-offlineRL-subopt}. We replace the dependence on \cref{prop:tv-high-prob-event-model-based} with \cref{prop:kl-high-prob-event-model-based}. 
    We then only have to take care of step $(g)$ in \cref{eq:mpqi-part-2}. We start from analyzing $(I)$ as in \cref{eq:mpqi-analyze-I}:
\begin{align*}
    &\E_{s_0\sim d_0} [\gamma \E_{s'\sim P^o_{s_0,\pi^*(s_0)}}(V^{\pihat^*}_{\cPhat}(s'))  - \gamma \E_{s'\sim P_{s_0,\pi^*(s_0)}}(V^{\pihat^*}_{\cPhat}(s'))] \nn\\
    &= \E_{s_0\sim d_0} [ \gamma \E_{s'\sim P^o_{s_0,\pi^*(s_0)}}(V^{\pihat^*}_{\cPhat}(s')) - \gamma \E_{s'\sim \Ptilde^o_{s_0,\pi^*(s_0)}}(V^{\pihat^*}_{\cPhat}(s')) \nn\\
    &\hspace{2cm} +\gamma \E_{s'\sim \Ptilde^o_{s_0,\pi^*(s_0)}}(V^{\pihat^*}_{\cPhat}(s'))  - \gamma \E_{s'\sim P_{s_0,\pi^*(s_0)}}(V^{\pihat^*}_{\cPhat}(s'))] \\
    &\stackrel{(a)}{\leq} \E_{s_0\sim d_0} [ \gamma \sqrt{2\ln(2) \dkl( P^o_{s_0,\pi^*(s_0)}, \Ptilde^o_{s_0,\pi^*(s_0)} )} \|V^{\pihat^*}_{\cPhat}(s')\|_\infty]  \nn\\
    &\hspace{2cm} +\gamma \E_{s'\sim \Ptilde^o_{s_0,\pi^*(s_0)}}(V^{\pihat^*}_{\cPhat}(s'))  - \gamma \E_{s'\sim P_{s_0,\pi^*(s_0)}}(V^{\pihat^*}_{\cPhat}(s'))] \\
    &\stackrel{(b)}{\leq}   \frac{2\gamma}{1-\gamma} \E_{s_0\sim d_0} \bigg[
    \sqrt{\frac{C\cardS \log(\cardS^2\cardA/\delta)\log(N)}{N(s_0,\pi^*(s_0))}}\indic\{N(s_0,\pi^*(s_0))\geq 1\} \bigg ]\nn\\
    &\hspace{2cm} +\gamma \E_{s_0\sim d_0} [\E_{s'\sim \Ptilde^o_{s_0,\pi^*(s_0)}}(V^{\pihat^*}_{\cPhat}(s'))  - \gamma \E_{s'\sim P_{s_0,\pi^*(s_0)}}(V^{\pihat^*}_{\cPhat}(s'))] \nn\\
    &\stackrel{(c)}{\leq}   \frac{4\gamma}{1-\gamma} \E_{s_0\sim d_0} \bigg[
    \sqrt{\frac{C\cardS \log(\cardS^2\cardA/\delta)\log(N)}{N(s_0,\pi^*(s_0))}}\indic\{N(s_0,\pi^*(s_0))\geq 1\} \bigg],
    \end{align*}
    where $(a)$ follows from Hölder's inequality and Pinsker's inequality \citep[Lemma 12.6.1]{cover1991information}, $(b)$ holds with probability at least $1-\delta$ by \cref{prop:kl-high-prob-event-model-based}, and $(c)$ again follows from Hölder's inequality and Pinsker's inequality under the definition of uncertainty set $\cPhat$. 
    Now combining and analyzing the rest of the steps as in the proof of \cref{thm:tv-tabular-offlineRL-subopt} completes the proof.
\end{proof}

We also provide a similar offline RL suboptimality guarantee below for the chi-square uncertainty set using \cref{prop:chi-square-high-prob-event-model-based}.
\begin{theorem}\label{thm:chi-square-tabular-offlineRL-subopt} Let $\pi_K$ be the DRQI policy  after $K$ iterations under the chi-square uncertainty set $\cPhat^\chisq$ (under add-$\log(1/\delta)$ estimator). With probability at least $1-\delta$ it holds that
    \begin{align*}
    \E_{s_0\sim d_0} [V^{\pi^*}(s_0) - \E_\cD [{V}^{\pi_K}(s_0)]] \leq  \frac{64\gamma \sqrt{C_{\pi^*}}}{(1-\gamma)^2}  \sqrt{\frac{C\cardS^2 \log(\cardS^2\cardA/\delta)}{N }} + \frac{2\gamma^{K+1}}{(1-\gamma)^2}.
\end{align*}
\end{theorem}
\begin{proof}
    The proof again follows exactly as in the proof of \cref{thm:tv-tabular-offlineRL-subopt}. We replace the dependence on \cref{prop:tv-high-prob-event-model-based} with \cref{prop:chi-square-high-prob-event-model-based}. 
    We again only have to take care of step $(g)$ in \cref{eq:mpqi-part-2}. We start from analyzing $(I)$ as in \cref{eq:mpqi-analyze-I}:
\begin{align*}
    &\E_{s_0\sim d_0} [\gamma \E_{s'\sim P^o_{s_0,\pi^*(s_0)}}(V^{\pihat^*}_{\cPhat}(s'))  - \gamma \E_{s'\sim P_{s_0,\pi^*(s_0)}}(V^{\pihat^*}_{\cPhat}(s'))] \nn\\
    &= \E_{s_0\sim d_0} [ \gamma \E_{s'\sim P^o_{s_0,\pi^*(s_0)}}(V^{\pihat^*}_{\cPhat}(s')) - \gamma \E_{s'\sim \Ptilde^o_{s_0,\pi^*(s_0)}}(V^{\pihat^*}_{\cPhat}(s')) \nn\\
    &\hspace{2cm} +\gamma \E_{s'\sim \Ptilde^o_{s_0,\pi^*(s_0)}}(V^{\pihat^*}_{\cPhat}(s'))  - \gamma \E_{s'\sim P_{s_0,\pi^*(s_0)}}(V^{\pihat^*}_{\cPhat}(s'))] \\
    &\stackrel{(a)}{\leq} \E_{s_0\sim d_0} \big[ 2\gamma \sqrt{ \dchi( P^o_{s_0,\pi^*(s_0)}, \Ptilde^o_{s_0,\pi^*(s_0)} )} \|V^{\pihat^*}_{\cPhat}\|_\infty \big]  \nn\\
    &\hspace{2cm} +\gamma [\E_{s'\sim \Ptilde^o_{s_0,\pi^*(s_0)}}(V^{\pihat^*}_{\cPhat}(s'))  - \gamma \E_{s'\sim P_{s_0,\pi^*(s_0)}}(V^{\pihat^*}_{\cPhat}(s'))] \\
    &\stackrel{(b)}{\leq}   \frac{2\gamma}{1-\gamma} \E_{s_0\sim d_0} \bigg[
    \sqrt{\frac{C\cardS \log(\cardS^2\cardA/\delta)}{N(s_0,\pi^*(s_0))}}\indic \{N(s_0,\pi^*(s_0))\geq 1\} \bigg] \nn\\
    &\hspace{2cm} +\gamma \E_{s_0\sim d_0} [\E_{s'\sim \Ptilde^o_{s_0,\pi^*(s_0)}}(V^{\pihat^*}_{\cPhat}(s'))  - \gamma \E_{s'\sim P_{s_0,\pi^*(s_0)}}(V^{\pihat^*}_{\cPhat}(s'))] \nn\\
    &\stackrel{(c)}{\leq}   \frac{4\gamma}{1-\gamma} \E_{s_0\sim d_0} \bigg[
    \sqrt{\frac{C\cardS \log(\cardS^2\cardA/\delta)}{N(s_0,\pi^*(s_0))}}\indic \{N(s_0,\pi^*(s_0))\geq 1\} \bigg],
    \end{align*}
    where $(a)$ follows from Hölder's inequality, and from Pinsker's inequality \citep[Lemma 12.6.1]{cover1991information} and \citep[Lemma 11.1]{basu2011statistical} we have $ \dtv(p,q)\leq 2\sqrt{\dchi(p,q)}$ for any two distributions, $(b)$ holds with probability at least $1-\delta$ by \cref{prop:chi-square-high-prob-event-model-based}, and $(c)$ follows same as $(a)$ but under the definition of uncertainty set $\cPhat$.
    Now combining and analyzing the rest of the steps as in the proof of \cref{thm:tv-tabular-offlineRL-subopt} completes the proof.
\end{proof}

\section{Results and Proofs of LM-DRQI} \label{appen:linear-mdp}

In the following, we always use $c>0$ for a small universal constant whose exact value might be changing.
We allow $\lambda=\Omega(1)$ but set $\lambda=1$ for simplicity.
In what follows, we use $\indic_i\in\R^{d\times 1}$ to denote vector with values $0$ except $1$ at position $i$.
We first make a similar observation as in \cref{prop:tv-high-prob-event-model-based}-\cref{prop:chi-square-high-prob-event-model-based} that the true model $P^o$ lies in the uncertainty set $\cPhat$ with high probability. We make this formal in the proposition below.
\begin{proposition}\label{prop:high-prob-event-linear-mdp}
We have $\nu^o\in\cMhat$  with probability at least $1-\delta$. Furthermore, $P^o\in\cPhat$ also holds with probability at least $1-\delta$.
\end{proposition}
\begin{proof}
Let $\indic_i\in\R^{d\times 1}$ denote vector with values $0$ except $1$ at position $i$ and $\indic(s_t')\in\R^{\cardS\times 1}$ denote vector with values $0$ except $1$ at position $s_t'$.
Fixing an $i\in[d]$ and $V\in\cV$, we have the following:
\begin{align}
    \E_{\nu^o_i}[V] - &\E_{\nuhat_i}[V] = (\nu^o_i)^\top V - (\nuhat_i)^\top V =  \indic_i^\top (\nu^o)^\top V - \indic_i^\top (\nuhat)^\top V \nn \\
    &\stackrel{(a)}{=} \indic_i^\top \Lambda_N^{-1} (\frac{\lambda}{N} I +  \frac{1}{N} \sum_{t=1}^N\phi(s_t,a_t)\phi(s_t,a_t)^\top) (\nu^o)^\top V - \indic_i^\top (\nuhat)^\top V \nn\\
    &\stackrel{(b)}{=} \frac{\lambda}{N} \indic_i^\top \Lambda_N^{-1} (\nu^o)^\top V   + \frac{1}{N}  \indic_i^\top \Lambda_N^{-1}  \sum_{t=1}^N\phi(s_t,a_t) (P^o_{s_t,a_t})^\top V - \indic_i^\top (\nuhat)^\top V
    \nn\\
    &\stackrel{(c)}{=} \frac{\lambda}{N} \indic_i^\top \Lambda_N^{-1} (\nu^o)^\top V   + \frac{1}{N}  \indic_i^\top \Lambda_N^{-1}  \sum_{t=1}^N\phi(s_t,a_t) (P^o_{s_t,a_t})^\top V - \frac{1}{N} \indic_i^\top  \Lambda_N^{-1} \sum_{t=1}^N  \phi(s_t,a_t) \indic(s_t')^\top V
    \nn\\
    &\stackrel{(d)}{=} \frac{\lambda}{N} \indic_i^\top \Lambda_N^{-1} (\nu^o)^\top V   + \frac{1}{N}  \indic_i^\top \Lambda_N^{-1}  \sum_{t=1}^N\phi(s_t,a_t) \epsilon_t^\top V , \label{eq:linear-mdp-HP-set-0}
\end{align} where $(a)$ is by $\Lambda_N = \frac{\lambda}{N} I +  \frac{1}{N} \sum_{t=1}^N\phi(s_t,a_t)\phi(s_t,a_t)^\top$, $(b)$ by $\phi(s_t,a_t)^\top) (\nu^o)^\top = P^o_{s_t,a_t}(\cdot)$, $(c)$ by 
$\nuhat(s') = \frac{1}{N} \Lambda_N^{-1} \sum_{t=1}^N  \phi(s_t,a_t) \indic\{s'=s_t'\}$, and $(d)$ by setting $\epsilon_t=(P^o_{s_t,a_t}-\indic(s_t'))$.

Before proceeding, here is a consequence of \cref{as:linear-MDP}. Consider any $(s,a)\in\ScA$. For any linear MDP $P_{s,a}(s') = \phi(s,a)^\top  \nu(s')$, summing both sides across $s'$, we get  $$1= \sum_{s'}P_{s,a}(s') = \phi(s,a)^\top  \sum_{s'}\nu(s') = \sum_{i\in[d]} \phi_i(s,a).$$ Since $\phi_i(s,a)\geq 0$ and $\|x\|_2\leq \|x\|_1$ for $x\in\R^d$, $\|\phi(s, a) \|_2 \leq 1$ follows.
Now we analyze the two terms in \cref{eq:linear-mdp-HP-set-0}.
First, \begin{align*}
     |\frac{\lambda}{N}  \indic_i^\top \Lambda_N^{-1} (\nu^o)^\top V| &\leq \frac{\lambda}{N}  \norm{  \indic_i^\top \Lambda_N^{-1} }_1 \norm{(\nu^o)^\top V}_\infty \\
    &\stackrel{(e)}{\leq} \frac{1}{1-\gamma} \frac{\lambda}{N}  \norm{  \indic_i^\top \Lambda_N^{-1} }_1 \\
    &\stackrel{(f)}{\leq} \frac{\sqrt{d}}{1-\gamma} \frac{\lambda}{N}  \norm{  \indic_i^\top \Lambda_N^{-1} }_2 \\
    &= \frac{\sqrt{d}}{1-\gamma} \frac{\lambda}{N} \sqrt{\indic_i^\top \Lambda_N^{-1}\Lambda_N^{-1}\indic_i }  \\
    &\stackrel{(g)}{\leq} \frac{\sqrt{d}}{1-\gamma} \frac{\lambda}{N} \sqrt{\norm{\Lambda_N^{-1}}_{\mathrm{op}}} \sqrt{\indic_i^\top \Lambda_N^{-1}\indic_i }  \\
    &= \frac{\sqrt{d}}{1-\gamma} \frac{\lambda}{N} \norm{\Lambda_N^{-1/2}}_{\mathrm{op}}  \norm{  \indic_i  }_{\Lambda_N^{-1}} \\
    &\stackrel{(h)}{\leq} \frac{1}{1-\gamma} \sqrt{\frac{d\lambda}{N}}   \sqrt{\Lambda_N^{-1}(i,i)} ,
\end{align*}
where $(e)$ follows since $V\in\cV=\{V(\cdot)=\max_{a} \phi^\top(\cdot,a) w :w\in\R^{d}, \|w\|_2\leq 1/(1-\gamma)\}$ satisfies $|V(s)| \leq \norm{\max_{a} \phi^\top(s,a)}_2\|w\|_2 \leq 1/(1-\gamma)$ for any $s\in\cS$ and $\nu^o$ is a probability distribution, $(f)$ by $\|x\|_1\leq\sqrt{d}\|x\|_2$ for $x\in\R^d$, 
$(g)$ by $x^\top A y \leq\|A\|_{\mathrm{op}} x^\top y$ for positive definite matrix $A$ with maximum eigenvalue $\|A\|_{\mathrm{op}}$ and for $x,y\in\R^d$,
and $(h)$ follows since $\Lambda_N$'s minimal absolute value is $\lambda/N$ in its diagonal entries.

Second, by Cauchy-Schwarz on $\Lambda_N^{-1}$-norm, \begin{align*}
     |\frac{1}{N}  \indic_i^\top \Lambda_N^{-1} \sum_{t=1}^N\phi(s_t,a_t) \epsilon_t^\top V| &\leq \frac{1}{N}   \norm{  \indic_i^\top  }_{\Lambda_N^{-1}}  \norm{ \sum_{t=1}^N\phi(s_t,a_t) \epsilon_t^\top V  }_{\Lambda_N^{-1}}  = \frac{1}{N} \sqrt{\Lambda_N^{-1}(i,i)} \norm{ \sum_{t=1}^N\phi(s_t,a_t) \epsilon_t^\top V  }_{\Lambda_N^{-1}}.
\end{align*}

We now get back to analyzing \cref{eq:linear-mdp-HP-set-0} using these intermediate steps.
Fix $i\in[d]$. For all $V\in\cV$, we have the following uniform bound: \begin{align*}
    &|\E_{\nu^o_i}[V] - \E_{\nuhat_i}[V]| \leq \frac{1}{1-\gamma} \sqrt{\frac{d\lambda}{N}}   \sqrt{\Lambda_N^{-1}(i,i)} + \frac{1}{N} \sqrt{\Lambda_N^{-1}(i,i)} \norm{ \sum_{t=1}^N\phi(s_t,a_t) \epsilon_t^\top V  }_{\Lambda_N^{-1}} \\
    &\stackrel{(i)}{\leq} \frac{1}{1-\gamma} \sqrt{\frac{d\lambda}{N}}   \sqrt{\Lambda_N^{-1}(i,i)} + \frac{1}{N}   \sqrt{\Lambda_N^{-1}(i,i)}  \cO\left( \frac{\sqrt{dN}\log(N/((1-\gamma)\delta))}{ 1-\gamma} \right)
    \\
    &\leq \frac{\cO( \log(N/((1-\gamma)\delta)) )}{1-\gamma} \sqrt{\frac{d}{N}}   \sqrt{\Lambda_N^{-1}(i,i)} ,
\end{align*} where $(i)$ holds with probability $1-\delta$ by \cref{lem:linear-MDP-uniform-concentrability}.

Let $c_1>0$ be some universal constant.
Furthermore, with an additional uniform bound,   the following holds for all $i\in[d]$ with probability at least $1-\delta$: \begin{align}
 \ipmV(\nu^o_i,\nuhat_i) \leq \frac{c_1 \log(Nd/((1-\gamma)\delta)) }{1-\gamma} \sqrt{\frac{d}{N}}   \sqrt{\Lambda_N^{-1}(i,i)}. \label{eq:linear-mdp-HP-set-1}
\end{align}

It is now straightforward to see $\nu^o\in\cMhat$ holds with probability at least $1-\delta$ by recalling:
\[\cMhat = \bigotimes_{i\in[d]} \cMhat_{i} \quad \text{where} \quad \cMhat_{i} = \bigg\{ \nu_i\in \Delta(\cS) : \ipmV(\nu_i,\nuhat_i) \leq \frac{c_1 \log(Nd/((1-\gamma)\delta)) }{1-\gamma} \sqrt{\frac{d}{N}}   \sqrt{\Lambda_N^{-1}(i,i)}  \bigg\}.\]

Furthermore, recall $P^o_{s,a}(s') = \sum_{i\in[d]} \phi_i(s,a) \nu^o_i(s')$ and $\Phat^o_{s,a}(s') = \sum_{i\in[d]} \phi_i(s,a) \nuhat_i(s')$.
We now have the following equations: \begin{align}
    & \sup_{V\in\cV}  \bigg|\int_{\cS} (P^o_{s,a} - \Phat_{s,a}) V(ds') \bigg| \nn 
    =   \sup_{V\in\cV}  \bigg|\int_{\cS} \sum_{i=1}^d \phi_i(s,a) (\nu^o_i(s') - \nuhat_i(s')) V(ds') \bigg|
    \\&=     \sup_{V\in\cV} \bigg| \sum_{i=1}^d \phi_i(s,a)  \int_{\cS}  (\nu^o_i(s') - \nuhat_i(s')) V(ds') \bigg| \leq  \sup_{V\in\cV}  \sum_{i=1}^d |\phi_i(s,a)|  |\int_{\cS}  (\nu^o_i(s') - \nuhat_i(s')) V(ds') | \nn \\
    &\leq     \sum_{i=1}^d |\phi_i(s,a)|  \sup_{V\in\cV}|\int_{\cS}  (\nu^o_i(s') - \nuhat_i(s')) V(ds') | \nn
    \\&=  \sum_{i=1}^d |\phi_i(s,a)| \cdot  \ipmV(\nu^o_i,\nuhat_i) \leq \frac{c_1 \log(Nd/((1-\gamma)\delta)) }{1-\gamma} \sqrt{\frac{d}{N}} \sum_{i=1}^d \norm{ \phi_i(s,a) \indic_i  }_{\Lambda_N^{-1}}  , 
    \label{eq:lin-mdp-hp-lem-1}
\end{align} where the last inequality follows by \cref{eq:linear-mdp-HP-set-1}.
This holds with probability at least $1-\delta$ for all $s,a$ together. Thus we have a high probability event that $P^o\in\cPhat$ with probability at least $1-\delta$. %This completes the proof of this lemma.
\end{proof}

Before presenting our main result we adapt \citep[Corollary 4.5]{jin2021pessimism} to present a high probability result adhering to the sufficient coverage assumption (\cref{assum:linear-MDP-sufficient-coverage}). 
\begin{lemma}\label{lem:data-inverse-matrix-approximation}
    For any $s,a$, we have with probability at least $1-\delta$ that  $\sum_{i\in[d]}\E_{s,a\sim d^{\pi^*}_{P^o}} [ \| \phi_i(s,a) \indic_i^\top \|_{\Lambda_N^{-1}} ] \leq \sqrt{   {\rank(\Sigma_{d^{\pi^*}_{P^o}})}/{\drectsuffcov   } }$
    where $\Lambda_N = \frac{\lambda}{N} I +  \frac{1}{N} \sum_{t=1}^N\phi(s_t,a_t)\phi(s_t,a_t)^\top$, $\Sigma_{d^{\pi^*}_{P^o}} =  \E_{s,a\sim d^{\pi^*}_{P^o}}\phi(s,a)\phi(s,a)^\top$.
\end{lemma}
\begin{proof}
This proof follows similar steps in the proof of \citep[Corollary 4.5]{jin2021pessimism}.
Firstly notice, \begin{align*}
    \sum_{i\in[d]}\E_{s,a\sim d^{\pi^*}_{P^o}} [ \| \phi_i(s,a) \indic_i \|_{\Lambda_N^{-1}} ] &=  \sum_{i\in[d]}\E_{s,a\sim d^{\pi^*}_{P^o}} [ \sqrt{ (\phi_i(s,a) \indic_i)^\top \Lambda_N^{-1} (\phi_i(s,a) \indic_i)} ] \\
    &=  \sum_{i\in[d]}\E_{s,a\sim d^{\pi^*}_{P^o}} [ \sqrt{ \trace((\phi_i(s,a) \indic_i)(\phi_i(s,a) \indic_i)^\top \Lambda_N^{-1}  )} ] \\
    &\stackrel{(a)}{\leq}  \sqrt{d} \sqrt{ \sum_{i\in[d]} \trace(\E_{s,a\sim d^{\pi^*}_{P^o}} [(\phi_i(s,a) \indic_i)(\phi_i(s,a) \indic_i)^\top] \Lambda_N^{-1}  )}  \\
    &\stackrel{(b)}{\leq}  \sqrt{d} \sqrt{ \sum_{i\in[d]} \trace( \Sigma^{i}_{d^{\pi^*}_{P^o}} \cdot ((1/N)I + \drectsuffcov \cdot d \cdot \Sigma^{i}_{d^{\pi^*}_{P^o}})^{-1}  )}  \\
    &\stackrel{(c)}{\leq}  \sqrt{d} \sqrt{ \sum_{i\in[d]}  \frac{\lambda^{i}_{d^{\pi^*}_{P^o}}}{(1/N) + \drectsuffcov \cdot d \cdot \lambda^{i}_{d^{\pi^*}_{P^o}}}  }  \\
    &\stackrel{(d)}{\leq}  \sqrt{d} \sqrt{   \frac{\rank(\Sigma_{d^{\pi^*}_{P^o}})}{(1/N) + \drectsuffcov \cdot d  } }  \leq \sqrt{   \frac{\rank(\Sigma_{d^{\pi^*}_{P^o}})}{ \drectsuffcov   } },
\end{align*}
where $(a)$ follows by $\|x\|_1\leq\sqrt{d}\|x\|_2$ for $x\in\R^d$ and Jensen's inequality, 
$(b)$ holds with probability at least $1-\delta$ by the sufficient coverage assumption (\cref{assum:linear-MDP-sufficient-coverage}),
$(c)$ follows by denoting eigenvalues $\lambda^{i}_{d^{\pi^*}_{P^o}}$ of rank-1 matrices $\Sigma^{i}_{d^{\pi^*}_{P^o}}$. For
$(d)$, we first notice \begin{align*}
    \Sigma_{d^{\pi^*}_{P^o}} =  \E_{s,a\sim d^{\pi^*}_{P^o}}\phi(s,a)\phi(s,a)^\top &= \E_{s,a\sim d^{\pi^*}_{P^o}}[\sum_{i,j\in[d]}(\phi_i(s,a) \indic_i)(\phi_j(s,a) \indic_j)^\top] \\
    &= \sum_{i\in[d]}\Sigma^{i}_{d^{\pi^*}_{P^o}} + \sum_{i,j\in[d]:i\neq j}\E_{s,a\sim d^{\pi^*}_{P^o}}[(\phi_i(s,a) \indic_i)(\phi_j(s,a) \indic_j)^\top] .
\end{align*} 
For any $k\in[d]$, let $\lambda_k$ denote $k^{\mathrm{th}}$ smallest eigenvalue of $\Sigma_{d^{\pi^*}_{P^o}}$. For any $k\in[d]$, we know from a fact of positive semidefinite matrices that $\lambda_k$ is at least as any $k^{\mathrm{th}}$ smallest eigenvalue of any matrix summand.
Moreover, since $\|\phi(s,a)\|_2\leq 1$, it follows by Jensen's inequality $\lambda_1 = \|\Sigma_{d^{\pi^*}_{P^o}}\|_{\mathrm{op}} \leq \E_{s,a\sim d^{\pi^*}_{P^o}}\|\phi(s,a)\phi(s,a)^\top\|_{\mathrm{op}}\leq 1$. Since $\Sigma_{d^{\pi^*}_{P^o}}$ is positive semidefinite, we have all $\lambda_k\in[0,1]$. Finally, step $(d)$ is concluded by the fact that the number of non-zero eigenvalues is equal to the rank of a positive semidefinite matrix. This completes the proof.
\end{proof}

We are now ready to present our main result of this linear MDP problem setting. With the above result, we now provide the offline RL suboptimality guarantee below.

\begin{theorem}\label{thm:linear-mdp-offlineRL-subopt} Let \cref{as:linear-MDP} hold. Let $\pi_K$ be the LM-DRQI algorithm policy after $K$ iterations. Then, under \cref{assum:linear-MDP-sufficient-coverage}, the following holds with probability at least $1-\delta$
    \begin{align*}
    &\E_{s_0\sim d_0} [V^{\pi^*}(s_0) - \E_\cD [{V}^{\pi_K}(s_0)]] \leq  \frac{2\gamma^{K+1}}{(1-\gamma)^2}   + \frac{c_1 \log(Nd/((1-\gamma)\delta)) }{(1-\gamma)^2} \sqrt{\frac{d\cdot\rank(\Sigma_{d^{\pi^*}_{P^o}})}{\drectsuffcov N}} .
\end{align*}
\end{theorem}
% \begin{theorem}\label{thm:linear-mdp-offlineRL-subopt} Let \cref{as:linear-MDP} hold. Let $\pi_K$ be the LM-DRQI algorithm policy after $K$ iterations. Then, \\
% (i) under \cref{assum:linear-MDP-sufficient-coverage}, the following holds with probability at least $1-\delta$
%     \begin{align*}
%     &\E_{s_0\sim d_0} [V^{\pi^*}(s_0) - \E_\cD [{V}^{\pi_K}(s_0)]] \leq  \frac{2\gamma^{K+1}}{(1-\gamma)^2}   + \frac{c_1 \log(Nd/((1-\gamma)\delta)) }{(1-\gamma)^2} \sqrt{\frac{d\cdot\rank(\Sigma_{d^{\pi^*}_{P^o}})}{\drectsuffcov N}} .
% \end{align*}
% (ii) with $C^\dagger_{\pi^*,\phi}<\infty$, the following holds with probability at least $1-\delta$
%     \begin{align*}
%     &\E_{s_0\sim d_0} [V^{\pi^*}(s_0) - \E_\cD [{V}^{\pi_K}(s_0)]] \leq  \frac{2\gamma^{K+1}}{(1-\gamma)^2}   + \frac{c_1 \log(Nd/((1-\gamma)\delta)) }{(1-\gamma)^2} \sqrt{\frac{d C^\dagger_{\pi^*,\phi} \rank(\Lambda)^2 \log(c/\delta)}{N}}  .
% \end{align*}
% \end{theorem}
\begin{proof}
We first recall our analyses of \cref{thm:tv-tabular-offlineRL-subopt}. We denote the value function of policy $\pi$ for the transition dynamics model $P$ as $V^\pi_P$. We now denote the robust value function \cite{panaganti22a,xu-panaganti-2023samplecomplexity,panaganti-rfqi} for uncertainty set $\cPhat$ as 
$V^{\pi}_{\cPhat} = \min_{P\in\cPhat} V^\pi_P$ and its optimal robust policy as $\pihat^* = \argmax_{\pi}V^{\pi}_{\cPhat} $. We let $Q^{\pi}_{\cPhat}$ be its corresponding robust Q-function. 
From robust RL \cite{panaganti22a,xu-panaganti-2023samplecomplexity,panaganti-rfqi} we can write the following robust Bellman equation: $Q^{\pi}_{\cPhat}(s,a)=r(s,a) + \gamma \min_{P_{s,a}\in\cPhat_{s,a}}\E_{s'\sim P_{s,a}}(V^{\pi}_{\cPhat}(s'))$. To make it notationally easy, we write $V^{\pi^*}$ ($d^\pi$) as $V^{\pi^*}_{P^o}$ ($d^\pi_{P^o}$) making the dependence on the model $P^o$ explicit.

We again recall \cref{eq:mpqi-part-0} in tandem with \cref{prop:high-prob-event-linear-mdp}:
\begin{align}
	\E_{s_0\sim d_0} [V^{\pi^*}_{P^o}(s_0) - {V}^{\pi_K}_{P^o}(s_0)] &\leq  \E_{s_0\sim d_0} [V^{\pi^*}_{P^o}(s_0) - V^{\pihat^*}_{\cPhat}(s_0)] + \frac{2\gamma^{K+1}}{(1-\gamma)^2}. \label{eq:lm-mpqi-part-0}
\end{align}
Further recalling \cref{eq:mpqi-part-1} we know,
\begin{align}
    &\E_{s_0\sim d_0} [V^{\pi^*}_{P^o}(s_0) - V^{\pihat^*}_{\cPhat}(s_0)] \leq \E_{s_0\sim d_0} [ 
    \gamma \E_{s'\sim P^o_{s_0,\pi^*(s_0)}}(V^{\pi^*}_{P^o}(s')-V^{\pihat^*}_{\cPhat}(s')) ]  \nn\\&\hspace{2cm}
    +\underbrace{ \E_{s_0\sim d_0} [\gamma \E_{s'\sim P^o_{s_0,\pi^*(s_0)}}(V^{\pihat^*}_{\cPhat}(s')) ] - \gamma \min_{P_{s_0,\pi^*(s_0)}\in\cPhat_{s_0,\pi^*(s_0)}}\E_{s'\sim P_{s_0,\pi^*(s_0)}}(V^{\pihat^*}_{\cPhat}(s'))}_{(I)} ]. \label{eq:lm-mpqi-part-1}
\end{align}

Analyzing $(I)$ in \cref{eq:lm-mpqi-part-1} for any $P\in\cPhat$:
\begin{align}
    (I) &= \E_{s_0\sim d_0} [ \gamma \E_{s'\sim P^o_{s_0,\pi^*(s_0)}}(V^{\pihat^*}_{\cPhat}(s')) - \gamma \E_{s'\sim \Phat^o_{s_0,\pi^*(s_0)}}(V^{\pihat^*}_{\cPhat}(s')) \nn \\
    &\hspace{2cm} +\gamma \E_{s'\sim \Phat^o_{s_0,\pi^*(s_0)}}(V^{\pihat^*}_{\cPhat}(s'))  - \gamma \E_{s'\sim P_{s_0,\pi^*(s_0)}}(V^{\pihat^*}_{\cPhat}(s'))] \nn \\
    &\stackrel{(g)}{\leq}   \frac{c_1 \log(Nd/((1-\gamma)\delta)) }{1-\gamma} \sqrt{\frac{d}{N}} \sum_{i=1}^d \norm{ \phi_i(s_0,\pi^*(s_0)) \indic_i  }_{\Lambda_N^{-1}}    \nn \\
    &\hspace{2cm} +\gamma \E_{s_0\sim d_0} [\E_{s'\sim \Phat^o_{s_0,\pi^*(s_0)}}(\Vhat^{\pihat^*}(s'))  -  \E_{s'\sim P_{s_0,\pi^*(s_0)}}(\Vhat^{\pihat^*}(s'))] \nn \\
    &\stackrel{(h)}{\leq}    \frac{2 c_1 \log(Nd/((1-\gamma)\delta)) }{1-\gamma} \sqrt{\frac{d}{N}} \sum_{i=1}^d \norm{ \phi_i(s_0,\pi^*(s_0)) \indic_i  }_{\Lambda_N^{-1}}   , \label{eq:lm-mpqi-part-2}
\end{align}
where $(g)$ holds with probability at least $1-\delta$, which follows from \cref{lem:linear-MDP-uniform-concentrability}, and $(h)$ follows by the definition of set $\cPhat$.

Substituting \cref{eq:lm-mpqi-part-2} back in  \cref{eq:lm-mpqi-part-1} and via recursion we get,
\begin{align*}
    &\E_{s_0\sim d_0} [V^{\pi^*}_{P^o}(s_0) - V^{\pihat^*}_{\cPhat}(s_0)] \\&\leq   \sum_{t=0}^\infty \gamma^t \E_{s\sim d^{\pi^*}_{P^o,t}} [ \frac{c_1 \log(Nd/((1-\gamma)\delta)) }{1-\gamma} \sqrt{\frac{d}{N}} \sum_{i=1}^d \norm{ \phi_i(s_0,\pi^*(s_0)) \indic_i  }_{\Lambda_N^{-1}} ] \\
    &= \frac{c_1 \log(Nd/((1-\gamma)\delta)) }{(1-\gamma)^2} \sqrt{\frac{d}{N}} \sum_{i=1}^d \E_{s\sim d^{\pi^*}_{P^o}} [ \norm{ \phi_i(s_0,\pi^*(s_0)) \indic_i  }_{\Lambda_N^{-1}}  ],
\end{align*}
where last equality follows by the definition of state-distribution $d^{\pi^*}_{P^o} = (1-\gamma) \sum_{t=0}^\infty \gamma^t d^{\pi^*}_{P^o,t}$. Now, putting this back in \cref{eq:lm-mpqi-part-0}, the offline RL guarantee becomes:
\begin{align}
    &\E_\cD [\E_{s_0\sim d_0} [V^{\pi^*}_{P^o}(s_0) - {V}^{\pi_K}_{P^o}(s_0)]] 
    \leq \frac{2\gamma^{K+1}}{(1-\gamma)^2}  \nn \\&\hspace{1cm}+\frac{c_1 \log(Nd/((1-\gamma)\delta)) }{(1-\gamma)^2} \sqrt{\frac{d}{N}} \sum_{i=1}^d \E_\cD[\E_{s\sim d^{\pi^*}_{P^o}} [ \norm{ \phi_i(s_0,\pi^*(s_0)) \indic_i  }_{\Lambda_N^{-1}}  ]] .\label{eq:linear-mdp-main-result-0} 
\end{align}

We now assume we have sufficient coverage \cref{assum:linear-MDP-sufficient-coverage} of linear MDP $P^o$.
Now under \cref{lem:relative-condition-data-inverse-matrix-approximation}, with probability at least $1-\delta$, from \cref{eq:linear-mdp-main-result-0}  we have \begin{align*}
    &\E_\cD [\E_{s_0\sim d_0} [V^{\pi^*}_{P^o}(s_0) - {V}^{\pi_K}_{P^o}(s_0)]] 
    \leq \frac{2\gamma^{K+1}}{(1-\gamma)^2}   \\&\hspace{1cm}+\frac{c_1 \log(Nd/((1-\gamma)\delta)) }{(1-\gamma)^2} \sqrt{\frac{d}{N}} \sum_{i=1}^d \E_\cD[\E_{s\sim d^{\pi^*}_{P^o}} [ \norm{ \phi_i(s_0,\pi^*(s_0)) \indic_i  }_{\Lambda_N^{-1}}  ]] \\
    &\leq \frac{2\gamma^{K+1}}{(1-\gamma)^2}   + \frac{c_1 \log(Nd/((1-\gamma)\delta)) }{(1-\gamma)^2} \sqrt{\frac{d\cdot\rank(\Sigma_{d^{\pi^*}_{P^o}})}{\drectsuffcov N}}. 
\end{align*}
This proves this result.
\end{proof}

Different from above, we now provide the offline RL suboptimality guarantee relying on the finite relative condition instead of the sufficient coverage assumption \cref{assum:linear-MDP-sufficient-coverage}.
Before presenting the result, here is another high probability result similar to \cref{lem:data-inverse-matrix-approximation} but now relies on the finite relative condition. 
\begin{lemma}\label{lem:relative-condition-data-inverse-matrix-approximation}
    Let $\lambda=1$.
    For any $s,a$, with probability at least $1-\delta$ we have $\sum_{i\in[d]}\E_{s,a\sim d^{\pi^*}_{P^o}} [\| \phi_i(s,a) \indic_i \|_{\Lambda_N^{-1}} \leq c \sqrt{C^\dagger_{\pi^*,\phi} \rank(\Lambda) (\rank(\Lambda) + \log(c/\delta))}$
    where $\Lambda_N = \frac{\lambda}{N} I +  \frac{1}{N} \sum_{t=1}^N\phi(s_t,a_t)\phi(s_t,a_t)^\top$,\\ $\Lambda =  \E_{s,a\sim\mu}\phi(s,a)\phi(s,a)^\top$, $ \drectrelativecondition  =  \max_{x\in\R^d} \sum_{i\in[d]} { d (x^\top \Sigma^{i}_{ d^{\pi^*}} x)}/{(x^\top \Lambda x)}.$
\end{lemma}
\begin{proof}
We follow the proof in \cref{lem:data-inverse-matrix-approximation} but use the relative condition number to get the required bound.
Firstly notice, \begin{align*}
    \sum_{i\in[d]}\E_{s,a\sim d^{\pi^*}_{P^o}} [ \| \phi_i(s,a) \indic_i \|_{\Lambda_N^{-1}} ] &=  \sum_{i\in[d]}\E_{s,a\sim d^{\pi^*}_{P^o}} [ \sqrt{ (\phi_i(s,a) \indic_i)^\top \Lambda_N^{-1} (\phi_i(s,a) \indic_i)} ] \\
    &=  \sum_{i\in[d]}\E_{s,a\sim d^{\pi^*}_{P^o}} [ \sqrt{ \trace((\phi_i(s,a) \indic_i)(\phi_i(s,a) \indic_i)^\top \Lambda_N^{-1}  )} ] \\
    &\stackrel{(a)}{\leq}  \sqrt{d} \sqrt{ \sum_{i\in[d]} \trace(\E_{s,a\sim d^{\pi^*}_{P^o}} [(\phi_i(s,a) \indic_i)(\phi_i(s,a) \indic_i)^\top] \Lambda_N^{-1}  )}  \\
    &=  \sqrt{d} \sqrt{ \sum_{i\in[d]} \trace( \Sigma^{i}_{d^{\pi^*}_{P^o}} \Lambda_N^{-1}  )}  \\
    &\stackrel{(b)}{\leq}  \sqrt{d} \sqrt{ \frac{\drectrelativecondition}{d} \trace( \Lambda \Lambda_N^{-1}  )  }  \\
    &=   \sqrt{  \drectrelativecondition \E_{s,a\sim \mu} [  \phi(s,a)^\top \Lambda_N^{-1} \phi(s,a)  ]} \\
    &\stackrel{(c)}{\leq}  c \sqrt{C^\dagger_{\pi^*,\phi} \rank(\Lambda) (\rank(\Lambda) + \log(c/\delta))},
\end{align*}
where $(a)$ follows by $\|x\|_1\leq\sqrt{d}\|x\|_2$ for $x\in\R^d$ and Jensen's inequality, 
$(b)$ follows by $\drectrelativecondition$ definition, and
$(c)$ holds by \cref{lem:data-generation-inequality} with probability at least $1-\delta$.
\end{proof}

\begin{corollary}\label{cor:linear-mdp-offlineRL-subopt} Let \cref{as:linear-MDP} hold. Let $\pi_K$ be the LM-DRQI algorithm policy after $K$ iterations. Then, with $C^\dagger_{\pi^*,\phi}<\infty$, the following holds with probability at least $1-\delta$
    \begin{align*}
    &\E_{s_0\sim d_0} [V^{\pi^*}(s_0) - \E_\cD [{V}^{\pi_K}(s_0)]] \leq  \frac{2\gamma^{K+1}}{(1-\gamma)^2}   + \frac{c_1 \log(Nd/((1-\gamma)\delta)) }{(1-\gamma)^2} \sqrt{\frac{d C^\dagger_{\pi^*,\phi} \rank(\Lambda)^2 \log(c/\delta)}{N}}  .
\end{align*}
\end{corollary}
\begin{proof}
The proof follows from \cref{thm:linear-mdp-offlineRL-subopt}.
In this corollary, we assume finite relative condition number $\drectrelativecondition < \infty$ for linear MDP $P^o$ instead of assuming \cref{assum:linear-MDP-sufficient-coverage}.
We also emphasize that in this result we only need to assume $ \Sigma^{i}_{d^{\pi^*}}$ for all $i\in[d]$, due to \cref{lem:relative-condition-data-inverse-matrix-approximation}, instead for all $ \Sigma^{(i,j)}_{d^{\pi^*}},$ $i,j\in[d]$ in \cref{as:linear-MDP}. Thus this result is more general than \cref{thm:linear-mdp-offlineRL-subopt}.
Now under \cref{lem:relative-condition-data-inverse-matrix-approximation}, with probability at least $1-\delta$, from \cref{eq:linear-mdp-main-result-0}  we have \begin{align*}
    &\E_\cD [\E_{s_0\sim d_0} [V^{\pi^*}_{P^o}(s_0) - {V}^{\pi_K}_{P^o}(s_0)]] 
    \leq \frac{2\gamma^{K+1}}{(1-\gamma)^2}   \\&\hspace{1cm}+\frac{c_1 \log(Nd/((1-\gamma)\delta)) }{(1-\gamma)^2} \sqrt{\frac{d}{N}} \sum_{i=1}^d \E_\cD[\E_{s\sim d^{\pi^*}_{P^o}} [ \norm{ \phi_i(s_0,\pi^*(s_0)) \indic_i  }_{\Lambda_N^{-1}}  ]] \\
    &\leq \frac{2\gamma^{K+1}}{(1-\gamma)^2}   +\frac{c_2 \log(Nd/((1-\gamma)\delta)) }{(1-\gamma)^2} \sqrt{\frac{d C^\dagger_{\pi^*,\phi} \rank(\Lambda)^2 \log(c/\delta)}{N}}  ,
\end{align*} where $c_2$ is a universal constant that only depends on $c_1$ and $c$ ($c$ is from \cref{lem:relative-condition-data-inverse-matrix-approximation}).
This completes the proof.
\end{proof}

In the following, we show that for a class of linear MDPs, the sufficient coverage assumption in \citet{jin2021pessimism} implies our sufficient coverage assumption (\cref{assum:linear-MDP-sufficient-coverage}) adapted from \citet{ma2022distributionally}.
\begin{lemma}\label{lem:suff-cov-equivalence-class}
    Consider a class of linear MDPs where $\Sigma^{i}_{d^{\pi^*}}=\Sigma^{j}_{d^{\pi^*}}$ for all $i,j\in[d]$.
    Define the random events $\cE_1=\{\omega:\Lambda_N(\omega) \geq I/N + \jinsuffcov \cdot \Sigma_{d^{\pi^*}}\}$ and $\cE_2=\{\omega:\Lambda_N(\omega) \geq I/N + \jinsuffcov \cdot d \Sigma^{i}_{d^{\pi^*}}\}$. Then we have $\cE_1 \subseteq \cE_2$.
\end{lemma}
\begin{proof}
    We know \begin{align*}
    \Sigma_{d^{\pi^*}} =  \E_{s,a\sim d^{\pi^*}}\phi(s,a)\phi(s,a)^\top &= \E_{s,a\sim d^{\pi^*}}[\sum_{i,j\in[d]}(\phi_i(s,a) \indic_i)(\phi_j(s,a) \indic_j)^\top] = \sum_{i\in[d]}\Sigma^{i}_{d^{\pi^*}} + \sum_{i,j\in[d]:i\neq j}\Sigma^{(i,j)}_{d^{\pi^*}} .
\end{align*} 
    Consider some non-zero $x\in \R^{d\times 1}$. Since $\Sigma^{(i,j)}_{d^{\pi^*}}$ are all positive semidefinite, we have \[ x^\top \Sigma_{d^{\pi^*}} x \geq \sum_{i\in[d]} x^\top\Sigma^{i}_{d^{\pi^*}}x = d (x^\top\Sigma^{i}_{d^{\pi^*}}x). \]
    Noting that $\cE_1=\{\omega:x^\top\Lambda_N(\omega) x \geq \|x\|_2^2/N + \jinsuffcov \cdot x^\top\Sigma_{d^{\pi^*}}x\}$ and $\cE_2=\{\omega:x^\top\Lambda_N(\omega)x \geq \|x\|_2^2/N + \jinsuffcov \cdot d x^\top\Sigma^{i}_{d^{\pi^*}}x\}$ finishes the proof.
\end{proof}

For a different class of linear MDPs we have the following.
\begin{lemma}\label{lem:relative-condition-equivalence-class}
    Consider a class of linear MDPs where $\Sigma^{i}_{d^{\pi^*}}=\Sigma^{(i,j)}_{d^{\pi^*}}$ for all $i,j\in[d]$.
    Let $ \relativecondition  =  \max_{x\in\R^d}  {  (x^\top \Sigma_{ d^{\pi^*}} x)}/{(x^\top \Lambda x)}$ and $ \drectrelativecondition  =  \max_{x\in\R^d} \sum_{i\in[d]} { d (x^\top \Sigma^{i}_{ d^{\pi^*}} x)}/{(x^\top \Lambda x)}$. Then we have $\drectrelativecondition = \relativecondition$.
\end{lemma}
\begin{proof}
    From \cref{lem:suff-cov-equivalence-class}, we already know 
    $\Sigma_{d^{\pi^*}} =  \sum_{i,j\in[d]}\Sigma^{(i,j)}_{d^{\pi^*}}.$
    Consider any non-zero $x\in \R^{d\times 1}$. From the class of linear MDPs, we further have \[ x^\top \Sigma_{d^{\pi^*}} x = \sum_{i\in[d]} d(x^\top\Sigma^{i}_{d^{\pi^*}}x). \]
    Now the statement directly follows.
\end{proof}

From \cref{cor:linear-mdp-offlineRL-subopt}, we get the offline suboptimality guarantee of the order $\frac{\sqrt{d C^\dagger_{\pi^*,\phi} \rank(\Lambda)^2}}{\sqrt{(1-\gamma)^4 N}}$ for LM-DRQI algorithm. Furthermore, under \cref{lem:relative-condition-equivalence-class}, it is comparable with \citet{uehara2021pessimistic} in \cref{tbl:linear-mdp-offline-rl-results-compare}.

\end{document}